\documentclass{article}

\usepackage{microtype}
\usepackage{graphicx}
\usepackage{multirow} 
\usepackage{subcaption}
\usepackage{booktabs} 
\usepackage[table]{xcolor} 
\usepackage{hyperref}
\usepackage{tabularx}

\usepackage[preprint]{icml2027}

\usepackage{amsmath}
\usepackage{amssymb}
\usepackage{mathtools}
\usepackage{amsthm}

\usepackage[capitalize,noabbrev]{cleveref}

\theoremstyle{plain}
\newtheorem{theorem}{Theorem}[section]
\newtheorem{proposition}[theorem]{Proposition}

\newtheorem{corollary}[theorem]{Corollary}
\theoremstyle{definition}

\newtheorem{assumption}[theorem]{Assumption}
\theoremstyle{remark}

\newcommand{\R}{\mathbb{R}}

\newcommand{\norm}[1]{\left\|#1\right\|}

\newcommand{\bfd}{\mathbf{d}}

\newcommand{\cJ}{\mathcal{L}}
\newcommand{\cN}{\mathcal{N}}
\newcommand{\cC}{\mathcal{C}}
\newcommand{\cK}{\mathcal{K}}
\newcommand{\diag}{\operatorname{diag}}

\renewcommand{\th}{\theta}

\usepackage[textsize=tiny]{todonotes}

\icmltitlerunning{Model Merging by Output-Space Projection}

\begin{document}

\twocolumn[
 \icmltitle{Model Merging by Output-Space Projection}



  \icmlsetsymbol{equal}{*}
  \begin{icmlauthorlist}
    \icmlauthor{Bethan Evans}{yyy}
    \icmlauthor{Benjamin Etheridge}{xxx}
    \icmlauthor{Stephen Roberts}{xxx}
    \icmlauthor{Jared Tanner}{yyy}
  \end{icmlauthorlist}
  \icmlaffiliation{yyy}{Department of Mathematics, University of Oxford, Oxford, UK}

  \icmlcorrespondingauthor{Bethan Evans}{bethan.evans@maths.ox.ac.uk}

  \icmlaffiliation{xxx}{Department of Engineering Science, University of Oxford, Oxford, UK}


  \icmlkeywords{Machine Learning, ICML}

  \vskip 0.3in
]



\printAffiliationsAndNotice{}  

\begin{abstract}
Model merging combines fine-tuned parameters into a single multi-task model without retraining. Existing methods, such as task arithmetic, model soups and DARE, are empirically effective but rely on heuristic merging rules and do not directly optimise the merged model against a specific objective. We introduce Output-Space Projection (OSP) merging, a general framework that formulates merging as an optimisation problem in output space and separates the choice of loss function from the choice of basis, which determines the admissible combinations of fine-tuned updates.
Under mean squared error and a first-order approximation, row-wise scaling of weight updates yields a convex quadratic programme (QP). Our theoretical analysis identifies the optimal subspace for a relaxed projection problem and derives a residual-energy diagnostic for basis selection. Experiments across vision and language benchmarks demonstrate competitive performance against existing methods and investigate the effects of basis choice and the number of merged layers on merging success.
\end{abstract}

\section{Introduction}
\label{sec:intro}
The proliferation of task-specific fine-tuned language and
vision models has motivated growing interest in model
merging: combining multiple fine-tuned checkpoints into
a single model capable of performing across tasks \citep{yang_model_2025}.
Existing approaches, including task arithmetic
\citep{ilharco2023editing}, model soups
\citep{wortsman2022model}, TIES \citep{yadav2023ties},
DARE \citep{yu_language_2024}, Fisher merging
\citep{matena_merging_2022}, and AdaMerging
\citep{yang_adamerging_2024}, combine task-specific parameter updates using fixed or stochastic rules and have demonstrated strong empirical performance.  
However, these methods do not characterise the optimal combination weights, nor have a precise mathematical notion or loss function for when their method is optimal against alternatives.

We formulate model merging as an output-space
optimisation problem, in which task-specific parameter
updates are combined to minimise a calibration loss.
This provides a general framework accommodating
different loss functions, layer-wise feature bases, and subsets
of network layers.
Within this framework, we first establish that
single-layer merging with row-wise diagonal masks
and squared-output loss can be formulated as a
convex quadratic programme (QP).
Solving this QP yields the optimal merging coefficients within the diagonal parameterisation that minimise a squared-output calibration objective using calibration inputs and fine-tuned model outputs, and subsumes existing diagonal-mask methods—including task arithmetic and model soup—as special cases. 

This projection perspective provides a unifying framework for a broad class of basis-dependent merging methods: rather than combining parameters directly, merging may be viewed as approximating the residual errors of a base model using a subspace of possible correction directions.  
The diagonal QP may be generalised to arbitrary
orthonormal feature bases, revealing an output-space
projection interpretation of model merging.
This allows us to derive a closed-form expression
for the irreducible approximation error associated
with a chosen output subspace, determined by the
fraction of residual output-space energy that it captures.
For the relaxed projection problem, the optimal
subspace is spanned by the leading eigenvectors
of the residual energy matrix.
This characterisation provides a theoretical basis
for selecting feature directions and quantifies the
potential improvement over diagonal merging.
Whilst the optimal basis can achieve lower squared-output
loss, the diagonal basis provides a simple alternative that avoids the need for a data-dependent basis construction. This motivates our investigation of how basis choice affects merging performance.

Finally, we extend the framework to alternative
calibration losses and investigate merging across
multiple model layers.
Our experiments indicate that merging only a small
number of the selected layers can retain most of the
performance improvements obtained through fine-tuning,
whereas merging all layers can introduce additional
errors through cross-layer interactions and the
limitations of local linearisation.

We release an implementation of our method.\footnote{\href{https://github.com/b-d-e/osp-merging}{https://github.com/b-d-e/osp-merging}}

\subsection{Related Work}
\label{sec:related_work}

\paragraph{Parameter-space merging:}
Early approaches to model merging combine the parameters
of independently trained or fine-tuned models directly.
Model soups \citep{wortsman2022model} average the weights
of models fine-tuned from a common initialisation, whereas
task arithmetic \citep{ilharco2023editing} instead
represents fine-tuning as a task vector, defined by the
difference between fine-tuned and pretrained parameters,
and combines these vectors through addition and scaling.
Directly combining task vectors can introduce
interference when different tasks produce conflicting
parameter updates. TIES-Merging \citep{yadav2023ties}
addresses this by trimming small-magnitude updates and resolves sign conflicts before merging. DARE \citep{yu_language_2024} further
investigates the redundancy of fine-tuned updates,
randomly dropping update parameters and rescaling
the remaining values before combining. Whilst these
approaches are computationally efficient, their merging
rules are primarily motivated by parameter-space
heuristics rather than direct optimisation of the
merged model's outputs.

\paragraph{Optimisation-based merging:}
Other methods derive merging rules from objectives that
account for differences in parameter importance or
activation statistics. 
Fisher-weighted averaging \citep{matena_merging_2022} motivates model merging from a Bayesian perspective by approximating the parameter posterior of each fine-tuned model using a Gaussian distribution centred at its learned parameters. Using a Laplace approximation, with the diagonal Fisher information as an estimate of posterior precision, the merged parameters are obtained through a precision-weighted average of the individual models. RegMean \citep{jin2023regmean} instead formulates
merging as a layer-wise least-squares problem that seeks to preserve
the intermediate representations of the fine-tuned
models.

\paragraph{Subspace and representation-based merging:}
Alternative methods exploit the structure of parameter updates, for example ZipIt! \citep{stoica_zipit_2024} aligns and combines features across neural
networks, allowing models with different initialisations
to be merged by accounting for correspondences between
their representations. KnOTS \citep{stoica_model_2024} uses singular value
decomposition to align task-specific low-rank updates
before applying existing merging methods, addressing
interference arising from misaligned update directions.
More generally, spectral approaches use matrix
decompositions to identify informative directions in
weight updates. These methods highlight the importance of each
direction in which parameter updates are combined.

Broadly, parameter-space methods employ predetermined merging rules; optimisation-based methods introduce objectives to guide the combination of model parameters; and subspace methods consider the directions along which updates are combined. OSP seeks to unify these perspectives by parameterising the merged updates in a chosen basis, and optimising the corresponding coefficients against a chosen output-space loss. The basis determines the admissible merging directions, whilst the loss determines how their contributions are evaluated on calibration data.




\section{Output Space Projection (OSP) Merging}
\label{sec:general_strategy}
Fix input dimension $d_0$ and output dimension $c$. 
Consider a pretrained model
$f(\cdot;\theta_0):\R^{d_0}\to\R^c$,
mapping inputs $x\in\R^{d_0}$ to outputs in
$\R^c$, with parameters
$\theta_0=\{W_\ell\}_{\ell=1}^{M}$,
where $W_\ell \in\R^{d_\ell\times d_{\ell-1}}$ denotes the weight matrix at layer $\ell$. The network is defined recursively by
\[
f(x;\theta_0)=W_Mh_{M-1}(x;\theta_0),
\]
with $h_0(x;\theta_0)=x$ and  for $\ell=1,\ldots,M-1,$
\begin{align}\nonumber
h_\ell(x;\theta_0)&=\sigma\left(W_\ell h_{\ell-1}(x;\theta_0)\right)
\end{align}
where \(\sigma\) denotes the activation function and the final layer is assumed to be linear. Including bias terms in each layer does not affect the theoretical results. Biases can be absorbed into the maps $h_{\ell}$, and are omitted from the notation for clarity.

We consider $K$ task-specific fine-tuned models
$f(x;\theta_0+\boldsymbol{\delta}^{(k)})$,
where
$\boldsymbol{\delta}^{(k)}
=(\delta_1^{(k)},\ldots,\delta_M^{(k)})$
denotes the collection of fine-tuned updates for task $k$,
with
\[
\delta_\ell^{(k)}=W_\ell^{(k)}-W_\ell.
\]
Here, $k\in\{1,\ldots,K\}$ indexes tasks and
$\ell\in\{1,\ldots,M\}$ indexes layers.
Our goal is to combine these updates into a single model
that approximates the task-specific outputs on a calibration
set, without further fine-tuning.

We consider merging an arbitrary subset of layers $\mathcal N\subseteq\{1,\ldots,M\}$.
Prior methods, such as model soups and task arithmetic, combine task-specific updates using scalar coefficients applied to each $\delta_\ell^{(k)}$.
A more flexible approach is to combine task-specific weight updates by independently scaling their rows. This corresponds to using the standard basis vectors $e_1,\ldots,e_{d_\ell}\in\mathbb R^{d_\ell}$ and for a task-specific update $\delta_\ell^{(k)}$, the product $e_i e_i^\top\delta_\ell^{(k)}$ isolates its $i$-th row, allowing each row to be scaled independently. We examine this diagonal merging formulation in Section~\ref{diagonal_qp}.

More generally, rather than restricting the merging directions to the standard basis, we consider an arbitrary set of orthonormal vectors $q_{\ell,1},\ldots,q_{\ell,p_\ell}\in\mathbb R^{d_\ell}$ for each layer $\ell\in\mathcal N$. Each vector $q_{\ell,i}$ extracts the component of a task-specific weight update $\delta_\ell^{(k)}$ along that direction through $q_{\ell,i}^\top\delta_\ell^{(k)}$. This component is then scaled by a merging coefficient $d_{\ell,k,i}$ and mapped back into the original feature space along $q_{\ell,i}$. Summing these contributions over all basis directions and task-specific updates, the pretrained weight matrix $W_\ell$ is replaced by
\[
W_\ell+
\sum_{k=1}^K\sum_{i=1}^{p_\ell}
d_{\ell,k,i}q_{\ell,i}q_{\ell,i}^\top\delta_\ell^{(k)}.
\]

Hence, collecting these direction-selecting vectors into a matrix with orthonormal columns
\[
Q_\ell :=
[q_{\ell,1},\ldots,q_{\ell,p_\ell}]
\in \mathbb R^{d_\ell\times p_\ell},
\]
we parameterise the merged residual update as
\begin{equation}
\delta_\ell(\mathbf d)
=
\sum_{k=1}^{K}
Q_\ell D_{\ell,k} Q_\ell^\top
\delta_\ell^{(k)},
\label{eq:general_merge}
\end{equation}
where
\[
D_{\ell,k}
=
\operatorname{diag}(\mathbf d_{\ell,k}),
\quad
\mathbf d_{\ell,k}
=
(d_{\ell,k,1},\ldots,d_{\ell,k,p_\ell})^\top
\in\mathbb R^{p_\ell}.
\]
Here, $d_{\ell,k,i}$ denotes the merging coefficient
for task $k$ along basis direction $i$ at layer $\ell$.
The columns of $Q_\ell$ determine the directions along which fine-tuned updates are independently scaled, while the coefficients determine the contribution of each task along those directions.
In particular, choosing $Q_\ell=I_{d_\ell}$ recovers standard
row-wise diagonal masking which we present as a simple example in Section \ref{diagonal_qp}.

Given calibration inputs $\{x_j\}_{j=1}^{n}$ and target outputs $\{y_j\}_{j=1}^{n}$,
which may be obtained from the task-specific fine-tuned
models, we formulate model merging as solving:
\begin{equation}\label{objective}
\boxed{
\min_{\mathbf d\in\mathcal C}
\sum_{j=1}^{n}
\ell\left(
f(x_j;\theta_0+\boldsymbol\delta(\mathbf d)),
y_j
\right)
}
\end{equation}
where $\mathbf d$ collects all merging coefficients,
$\boldsymbol\delta(\mathbf d)$ collects the merged updates
across all layers in $ \cN$, with zero updates elsewhere,
$\ell(\cdot,\cdot)$ is a chosen output-space loss, and
$\mathcal C$ denotes any constraints imposed on the
merging coefficients.

This formulation separates two central design choices: first
\emph{the loss function}, which determines how closely
the merged model should match the task-specific outputs,
and secondly \emph{the basis}, which determines directions along which fine-tuned updates can be independently scaled. Together with the constraints on the merging coefficients, the basis determines the set of admissible merged updates.
We investigate squared error and KL divergence as objective choices, and compare
general orthonormal bases with the computationally
efficient row-feature basis.

\subsection{Simplest Example: Single-layer diagonal QP}\label{diagonal_qp}
An important special case arises when merging is restricted to a single layer $\mathcal N=\{N\}$ using the standard basis $Q_N=I_{d_N}$  and the mean squared error loss. In this case, the columns of $Q_N$ are the standard basis vectors, so each direction corresponds to a row of the weight matrix, allowing independent row-wise scaling of the fine-tuned weight updates.
The merged parameter update in layer $N$ becomes
\begin{equation}\nonumber
\delta_N(\mathbf d)
=
\sum_{k=1}^{K}D_{N,k}\delta_N^{(k)},
\label{eq:diagonal_merge}
\end{equation}
where the diagonal matrices $D_{N,k}$ act from the left, and each diagonal entry 
scales  a row of each update\footnote{We could instead consider applying the diagonal mask from the right, corresponding to scaling the columns of the fine-tuned parameter update. Right multiplication controls how strongly each input feature contributes through the update whereas row scaling controls how much the update is applied to one output neuron. Neither approach is inherently better; the choice depends on whether the merging process is intended to operate at the level of input features or output neurons.}.

For calibration input $x_j$, let
\(
z_j := h_{N-1}(x_j;\th_0)\in\mathbb{R}^{d_{N-1}}
\)
denote the input to layer $N$, which absorbs all computation before layer $N$.
The merged model replaces $W_N$ by $W_N+\delta_N(\bfd)$, and for squared-output
loss on the calibration set $\cK$ we would like to minimise
\begin{equation}
  \cJ_I(\bfd) \;:=\; \sum_{j\in\cK} \norm{ f(x_j; \th_0 + \delta_N(\bfd)) - y_j }^2.
\end{equation}

\paragraph{Linear approximation}
To obtain a QP, we approximate the model output to first order in
the update $\delta_N(\bfd)$. Let $J_N(x_j)\in\mathbb{R}^{c\times d_N}$ denote the Jacobian
of the network output with respect to the pre-activation of layer $N$, evaluated at
$W_N z_j$  ; that is, the Jacobian of the network after
layer $N$, which we refer to as the downstream map. Then
\begin{equation}\label{eq:linearisation}
  f(x_j; \th_0 + \delta_N(\bfd))
  \;\approx\;
  f(x_j;\th_0) + J_N(x_j)\,\delta_N(\bfd)\,z_j .
\end{equation}
Only the computation after layer $N$ is linearised; any nonlinearities before layer $N$
are absorbed into $z_j$. The approximation is exact when the computation after
layer $N$ is linear; in particular, for final-layer merging $J_N(x_j) = I$. We use this
linearisation in the analysis that follows.

Using this linearisation
of the model after layer $N$ yields:
\begin{equation}
  \cJ_{I}(\bfd) \;: =\; \sum_{j\in\cK} \norm{A^{(j)}\bfd + b_j}^2,
\end{equation}
where $b_j := f(x_j;\theta_0)-y_j$ and $A^{(j)}$ is formed as follows: the update activation $r_{k,j} := \delta^{(k)}_N z_j \in \mathbb{R}^{d_N}$ is diagonalised to form $A^{(k,j)} := J_N(x_j)\diag(r_{k,j})$, and these are concatenated as $A^{(j)} := [A^{(1,j)}\;\cdots\;A^{(K,j)}] \in \mathbb{R}^{c \times K d_N}$. Here $y_j$ denotes either labelled calibration targets, when available, or outputs from the corresponding fine-tuned model.

Expanding gives
the convex quadratic
\begin{align}
  \label{eq:QP}
  \cJ_{I}(\bfd) \;&=\; \tfrac{1}{2}\bfd^\top H\bfd + g^\top\bfd + \mathrm{const},
  \end{align}
where we define
\begin{align}\label{eq:H-g}
  H = 2\sum_{j\in\cK} (A^{(j)})^\top A^{(j)}, \quad
  g = 2\sum_{j\in\cK} (A^{(j)})^\top b_j.
\end{align}
Subject to any convex constraint set $\cC$ (e.g.\ box constraints
$0\le\bfd\le\mathbf{1}$), the resulting programme
\begin{equation}
  \tag{QP}
  \min_{\bfd\in\cC}\;\tfrac{1}{2}\bfd^\top H\bfd + g^\top\bfd
\end{equation}
is convex and can be solved globally \citep{boyd_convex_2004} to obtain the diagonal mask vector $\mathbf d^\ast$. When $\mathcal C$ is convex, any attained
local minimum is globally optimal for the linearised
calibration objective within the diagonal-mask
parameterisation.  App.~\ref{app} contains analysis of how the fine-tuned updates and inputs influence the solution.


Existing merging methods can be recovered as particular
coefficient choices (using an appropriate convex constraint set $\mathcal{C}$): $D_{N,k}=\lambda_k I$ gives task
arithmetic \citep{ilharco2023editing}, while
$D_{N,k}=\frac{1}{K}I$ gives uniform model soups
\citep{wortsman2022model}.
Row-uniform DARE \citep{yu_language_2024} corresponds to
stochastic diagonal masks
$D_{N,k}=\operatorname{diag}(m^{(k)}/\rho)$, where
$m_i^{(k)}\sim\mathrm{Bernoulli}(\rho)$.
TIES \citep{yadav2023ties} is related through its
sparsity and sign-based restrictions, although its
general element-wise operations are not exactly
captured by row-wise masking. Table~\ref{table} gives more details.
Unlike these heuristic choices, our QP optimises
the coefficients directly against the calibration loss.
 Since task arithmetic, model soups, and row-wise DARE correspond to feasible points of the QP, the OSP solution achieves no worse linearised approximation of the calibration loss than these methods.  


\section{Choice of Basis}
\label{sec:basis}
The diagonal merging introduced in Section~\ref{diagonal_qp} corresponds to restricting updates to the standard coordinate basis $\{e_i\}_{i=1}^{d_N}$, corresponding to scaling the rows of the fine-tuned updates. In particular, diagonal masking allows independent scaling of each coordinate but does not permit mixing between coordinates, thereby restricting the admissible subspace to axis-aligned directions.  In this case, the corresponding output subspace is given by $\mathrm{span}\{J_N(x_j) e_i\}$.

We now investigate how the
choice of basis affects the attainable merging error when squared error loss is used.
As discussed in the general formulation of our problem in Section \ref{sec:general_strategy}, the choice of orthonormal vectors determines which feature directions of the
fine-tuned residual updates can be independently
scaled, and therefore which corrections the merged
model can represent. We first extend the squared-output
QP to arbitrary orthonormal bases, then characterise
the optimal subspace in a relaxed output-space
projection problem, and finally consider the standard
row-feature basis as a computationally efficient
special case.

\subsection{General Basis with Mean Squared Error}
\label{sec:general_basis}

Consider merging a single layer $N$ in an arbitrary
orthonormal basis
$Q_N=[q_1,\ldots,q_p]\in\mathbb R^{d_N\times p}$,
where $d_{N}$ is the number of output neurons at layer $N$, $p\leq d_N$ and
$Q_N^\top Q_N=I$. Following Section~\ref{sec:general_strategy},
the merged update is
\[
\delta_N(\mathbf d)
=
\sum_{k=1}^{K}
Q_N D_{N,k}Q_N^\top\delta_N^{(k)}.
\label{eq:basis_merge}
\]
For calibration example $j$, define the projected update activation of model $k$
along direction $q_a$ and the corresponding output-space correction direction,
\[
\alpha_{j,k,a} := q_a^\top\delta_N^{(k)} z_j,
\qquad
v_{j,a} := J_N(x_j)\,q_a\in\mathbb R^c .
\]
Under the linearisation~\eqref{eq:linearisation}, the merged update changes the
output by
\[
A^{(j)}\bfd = \sum_{k=1}^K\sum_{a=1}^p d_{N,k,a}\,\alpha_{j,k,a}\,v_{j,a} ,
\]
where $A^{(j)}\in\mathbb R^{c\times Kp}$ is now given by
\begin{equation}\label{eq:A-general}
A^{(j)} = \big[A^{(1,j)}\;\cdots\;A^{(K,j)}\big],
\end{equation}
with
\(
  A^{(k,j)} = \big[\alpha_{j,k,1}v_{j,1}\;\cdots\;\alpha_{j,k,p}v_{j,p}\big]\in\mathbb R^{c\times p}.\)

\begin{proposition}[General-basis QP]\label{thm:general-qp}
For single-layer merging in an orthonormal basis $Q_N\in\mathbb R^{d_N\times p}$,
the linearised squared-output objective
\[
\cJ_Q(\bfd)=\sum_{j\in\cK}\|A^{(j)}\bfd+b_j\|_2^2
\]
is the convex quadratic with $H$ and $g$ given by~\eqref{eq:H-g}, using the
general $A^{(j)}$ from~\eqref{eq:A-general}. In the unconstrained case:
\begin{enumerate}
  \item if $H\succ 0$, the minimiser is unique, $\bfd^\star=-H^{-1}g$;
  \item if $H$ is singular, the minimisers are
  $\{-H^{+}g+w : w\in\ker H\}$.
\end{enumerate}
\end{proposition}
Taking $Q_N=I$ (so $p=d_N$) recovers Section~\ref{diagonal_qp}, showing that
diagonal merging is a special case of the general-basis formulation.

Note that $H$ has terms consisting of a scalar multiplied by $q_a^\top (J_N(x_j))^\top J_N(x_j)q_{a'}$.
The matrix \((J_N(x_j))^\top J_N(x_j)\) describes how perturbations at layer \(N\) affect the model output. Although the basis directions \(q_a\) are orthogonal in the layer's feature space, their corresponding output corrections \(J_N(x_j)q_a\) need not be orthogonal. Consequently,  in the squared-output loss case different basis directions can interact, giving rise to nonzero off-diagonal entries in \(H\). The merging coefficients are therefore generally coupled and must be optimised jointly.

\subsection{Irreducible Error and Projection}
We now analyse a relaxed pointwise output-space projection problem to characterise the optimal output subspace and the minimum attainable squared-output error. Unlike the merging basis \(Q_N\), which specifies directions in the layer's feature space, the optimal subspace is defined directly in the model's output space \(\mathbb R^c\).
In this relaxed setting, each calibration residual $b_j$ may be corrected by an arbitrary vector in a common output subspace $\mathcal S \subseteq \mathbb R^c$. 
This yields a closed-form characterisation of how much residual energy a chosen output subspace can capture. 
The coefficient-tied problem induced by a fixed set of model updates is then handled by the general-basis QP in Proposition \ref{thm:general-qp}. 


\begin{theorem}[Merging as a Projection and Basis-Dependent Error]\label{projection}
Let $\mathcal K$ be the calibration set with $n := |\mathcal K|$, and let
$f(x_j;\theta_0)$ denote the base model with residual error
\[
b_j = f(x_j;\theta_0) - y_j \in \mathbb{R}^c.
\]
Define the matrix
\[
S := \sum_{j\in\mathcal K} b_j b_j^\top \in \mathbb{R}^{c\times c}.
\]
Let $\mathcal S\subseteq\mathbb{R}^c$ be a subspace with orthogonal projector
$P_{\mathcal S}$. In the relaxed problem where each residual $b_j$ is corrected
independently by a vector in $\mathcal S$, the optimal correction is
$-P_{\mathcal S}b_j$, hence the residual after merging for sample $j$ is
\[
(I-P_{\mathcal S})b_j.
\]
Using the identity $\sum_{j \in \mathcal K} \|P_{\mathcal S} b_j\|^2
= \operatorname{tr}(S P_{\mathcal S})$, the minimal relaxed squared-output loss on the
calibration set $\mathcal K$ attainable within $\mathcal S$ is
\begin{equation}
  \label{eq:error-trace}
  \cJ_{\mathcal S}
  \;=\;
  \sum_{j\in\mathcal K}\|b_j\|^2
  \;-\;
  \operatorname{tr}(S\,P_{\mathcal S}).
\end{equation}
\end{theorem}
Proof is provided in Appendix~\ref{prof_proof}.

Since \(\operatorname{tr}(S) = \sum_{j \in \mathcal K} \|b_j\|^2\) is the total residual energy, we can normalise \eqref{eq:error-trace} to obtain
\begin{equation}
   \frac{\cJ_{\mathcal S}}{\operatorname{tr}(S)}
= 1 - \frac{\operatorname{tr}(S P_{\mathcal S})}{\operatorname{tr}(S)}.
\end{equation}
Thus, \(\operatorname{tr}(S P_{\mathcal S})/\operatorname{tr}(S)\) is a data-dependent diagnostic of how well an output subspace captures residual energy. Since \(\sum_j\|b_j\|^2\) is independent of the choice of \(\mathcal S\), selecting the best output subspace of fixed dimension reduces to maximising the captured energy \(\sum_j\|P_{\mathcal S} b_j\|^2\). If $\operatorname{range}(A^{(j)})\subseteq\mathcal S$ for all $j$ (e.g.\ final-layer
merging with $\mathcal S=\operatorname{span}\{q_1,\dots,q_p\}$), every coefficient
vector produces corrections in $\mathcal S$, so $\cJ_Q(\mathbf d)\ge\cJ_{\mathcal S}$:
the captured-energy fraction is then an upper bound on the reduction achievable by
the QP method. This upper bound is tight only when a single feasible $\mathbf d$ realises every per-sample
projection.
\begin{figure}[b]
        \centering
        \includegraphics[width=.9\linewidth]{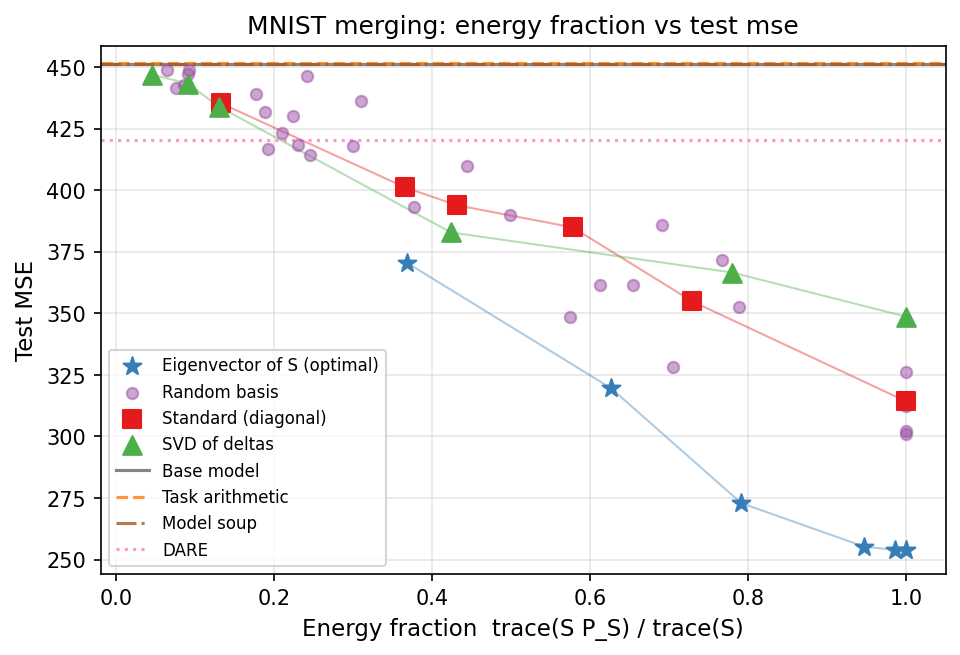}
        \caption{MSE against captured energy for increasing basis dimensions $p$ on a 5 layer ReLU network trained with MNIST.}
        \label{fig:energy_mnist}
\end{figure}
\begin{corollary}\label{cor}
Minimising the loss over all admissible $p$-dimensional subspaces
$\mathcal S \subseteq \mathbb{R}^c$ is equivalent to maximising $\operatorname{tr}(S P_{\mathcal S})$.
Among all $p$-dimensional subspaces of $\mathbb{R}^c$, the relaxed-optimal output subspace is given by the span of the top-$p$ eigenvectors of the residual matrix $S$.
\end{corollary}
\begin{proof}
Follows from the Ky-Fan variational characterisation of trace maxima of symmetric PSD matrices over rank-$p$ projectors \citep{fan1949weyl, eckart_approximation_1936}.
\end{proof}
The relaxed-optimal subspace identified by Corollary \ref{cor} is defined over all subspaces of $\mathbb{R}^c$. In practice, however, the model parametrisation restricts the admissible corrections to a subset of this space that can be realised by the model parametrisation in the chosen basis.

In our setting, a chosen basis $\{q_a\}_{a=1}^p$ in layer-$N$ feature-space induces the pointwise output subspace
\[
\mathcal S_{\mathrm{model}}(x_j) := \mathrm{span}\{J_N(x_j) q_1, \ldots, J_N(x_j)  q_p\}.
\]
This subspace describes the output directions available to corrections expressed in the chosen basis. The projection analysis above therefore provides a basis-dependent measure of how much residual energy these directions can capture. We note that $J_N(x_j)$ is constant at the final layer or when the downstream map is linear. 

In general, however, $\mathcal S_{\mathrm{model}}$ may fail to contain the principal residual directions identified by the projection analysis. The component of the residual energy lying outside this pointwise output subspace therefore gives a basis-dependent irreducible error for the relaxed projection problem.

Diagonal masking is optimal for the relaxed projection problem when the output subspace 
$\mathrm{span}\{J_N(x_j)e_1 , \ldots , J_N(x_j) e_p \}$ coincides with the top-$p$ eigenspace of $S$.
In the special case $J_N(x_j) =I$ (for example, when merging in the final layer), this reduces to selecting the $p$ coordinates with largest diagonal entries when $S$ is diagonal in the standard basis.

The excess squared-output loss on the calibration set $\mathcal K$ incurred by merging using a $p$-dimensional subspace $\mathcal{S}_{\mathrm{model}}$ instead of the optimal subspace $\mathcal{S}_{\mathrm{opt}}$ (when $J_N(x_j) $is constant) is given by
\begin{equation}\label{gap}
  \cJ_{\mathcal{S}_{\mathrm{model}}} - \cJ_{\mathcal{S}_{\mathrm{opt}}}
=
\operatorname{tr}\big(S (P_{\mathrm{opt}} - P_{\mathrm{model}})\big).  
\end{equation}


This quantity is computable from the calibration set and provides a data-dependent measure of the suboptimality of the chosen basis. 

Figure~\ref{fig:energy_mnist} plots the fraction of captured residual energy,
\(
\operatorname{tr}(S P_{\mathcal S}) /\operatorname{tr}(S),
\)
against MSE as the basis dimension $p$ increases. As $p$ increases, the captured residual energy increases and the MSE decreases, consistent with the projection analysis.
Figure~\ref{fig:energy_mnist} also compares different bases on five fine-tuned 3-layer ReLU networks trained on separate MNIST digit classification tasks, merging at the final layer. In Figure~\ref{energy} (App.~\ref{vit}), we provide results for the same experiment on ViT. 
For each basis, we restrict merging to the same dimension $p$ and solve the corresponding QP. The relaxed-optimal basis consistently achieves the lowest MSE, while random bases perform substantially worse. The standard diagonal basis is less effective than the relaxed-optimal basis, but remains competitive against the benchmark merging heuristics tested, indicating that the diagonal QP provides an efficient proxy when computing the relaxed-optimal basis, given by the leading eigenvectors of $S$, is impractical. The gap in Equation~\eqref{gap} provides a computable measure of this loss which quantifies the trade-off between computational efficiency and optimality.

\paragraph{Interpretation}
Geometrically, the residuals define an anisotropic energy distribution in output space, whose principal directions are given by the eigenvectors of the residual energy matrix $S$. The model merging objective therefore corresponds to selecting a subspace that captures as much of this energy as possible, with the optimal solution aligning with these principal directions. This perspective identifies the optimal output subspace for the relaxed projection problem and characterises the irreducible approximation error associated with restricting corrections to a chosen output subspace. However, we note that the merging basis must be selected in the layer's feature space, and the resulting corrections need not realise the relaxed-optimal projections. The analysis therefore provides a theoretical benchmark against which the approximation capacity of different merging bases can be assessed.


\paragraph{Iterative Basis Refinement}
The basis can also be refined iteratively by introducing additional directions that are orthogonal to the existing basis and selected to reduce the remaining output residuals. The merging coefficients are then re-optimised over the enlarged basis. Since the previous solution remains feasible by assigning zero coefficients to the new directions, the minimum achievable linearised calibration loss cannot increase.

\section{Multiple Layers and Alternative Losses}
\label{sec:loss}

The general formulation in Section~\ref{sec:general_strategy}
allows model merging to be optimised under different
output-space losses. So far we have discussed the special case of mean squared error evaluated per layer which leads to a convex quadratic
programme (QP). Now we consider two different choices: mean squared error solved simultaneously in every layer; and KL divergence, which directly
measures discrepancies between the predictive
distributions of the merged and fine-tuned models.

\subsection{Squared Output Loss}
\label{sec:mse_loss}
In Section~\ref{sec:general_basis}, we showed that
restricting merging to a single layer, under a
first-order approximation of the downstream mapping
and a squared-output loss, yields a convex quadratic
programme which may be solved exactly in the chosen basis.
Extending this approach to multiple layers introduces
additional interactions, since updates in earlier layers
affect the activations received by subsequent layers and so the objective is generally non-convex in the merge coefficients.

One approach is to merge layers sequentially,
solving a separate quadratic programme for each layer
while keeping the remaining layers fixed. However,
this greedy procedure does not generally guarantee
a globally optimal solution, since the coefficients
optimised at one layer affect the optimal coefficients
at subsequent layers. Furthermore, errors introduced
by each local linear approximation may accumulate as
additional layers are merged.

Alternatively, we may optimise the merging coefficients
across all chosen layers jointly by directly minimising
the squared-output loss. Since the merged updates are differentiable with respect
to the merging coefficients, this objective can be
optimised using gradient-based methods, such as gradient
descent or Adam, while keeping the pretrained and
fine-tuned weight matrices fixed. This avoids the need
to linearise the network and allows interactions between
updates at different layers to be accounted for during
optimisation. However, the resulting objective is
generally non-convex, so the global optimality guarantees
of the linearised quadratic formulation no longer apply.

In Section~\ref{exp}, we empirically compare solving the quadratic programme vs. joint gradient-based optimisation, finding that the latter achieves better merging performance at lower computational cost in the evaluated settings. 


\subsection{KL Divergence}
\label{sec:kl_loss}
While squared-output loss provides a convex quadratic formulation with the projection interpretation, alternative losses can better reflect the probabilistic structure of model predictions. Recent work by \citet{dalili_model_2025} formulates model merging as multi-teacher knowledge distillation, minimising KL divergence between the predictive distributions of fine-tuned models and the merged model. We consider minimising the KL divergence between the specialist fine-tuned model distribution and the basis-dependent OSP merged model, and relate this to Fisher merging. Under a fixed first-order approximation, the model logits are affine in the merging coefficients. Consequently, the resulting KL divergence objective is convex in these coefficients, although it is no longer quadratic and does not admit the same Euclidean projection characterisation as squared-error merging.

Let 
the target distribution be
$p_j(\delta^{(k)}) = \operatorname{softmax}
\left(f(x_j ; \th_0 + \delta^{(k)}) \right)$ from the corresponding
task-specific fine-tuned model.
The predictive distribution of the merged model
under the linearised logit approximation is
\begin{equation}\nonumber
q_j(\mathbf d)
=
\operatorname{softmax}
\left(f(x_j; \th_0) +A^{(j)}\mathbf d\right).
\label{eq:merged_distribution}
\end{equation}
We minimise the teacher-to-student KL divergence, where $\mathcal{K}_k$ is the data in the calibration set of model $k$:
\begin{equation}\nonumber
\cJ_{\mathrm{KL}}(\mathbf d)
=
\sum_{k=1}^{K} \sum_{j \in \mathcal{K}_k}
D_{\mathrm{KL}}
\left(p_j(\delta^{(k)})\|q_j(\mathbf d)\right),
\label{eq:kl_objective}
\end{equation}
with
\(
D_{\mathrm{KL}}(p\|q)
:=
\sum_{r=1}^{c}
p_{r}\log( p_{r}/q_{r}).
\)
Substituting the softmax parameterisation and
discarding terms independent of $\mathbf d$
gives
\begin{equation}\nonumber
\begin{aligned}
\cJ_{\mathrm{KL}}(\mathbf d)
=
&\sum_{k=1}^{K} \sum_{j \in \mathcal{K}_k}
\Bigg[
\log\sum_{r=1}^{c}
\exp\left( f(x_j; \th_0)_r+
(A^{(j)}\mathbf d)_r\right)\\
&-(p_j(\delta^{(k)}))^\top(f(x_j; \th_0)+A^{(j)}\mathbf d)
\Bigg]
+\mathrm{const}.
\end{aligned}
\label{eq:kl_logsumexp}
\end{equation}

Unlike squared error, this objective is not
quadratic however it remains convex in
$\mathbf d$.
In particular, its gradient is
\begin{equation}\nonumber
\nabla_{\mathbf d}\cJ_{\mathrm{KL}}
=
\sum_{k=1}^{K} \sum_{j \in \mathcal{K}_k}
(A^{(j)})^\top(q_j(\mathbf d)-p_j(\delta^{(k)})),
\label{eq:kl_gradient}
\end{equation}
and its Hessian is
\begin{equation}\nonumber
\nabla_{\mathbf d}^2\cJ_{\mathrm{KL}}
=
\sum_{k=1}^{K} \sum_{j \in \mathcal{K}_k}
(A^{(j)})^\top
\left[
\operatorname{diag}(q_j)
-q_jq_j^\top
\right]A^{(j)}
\succeq0.
\label{eq:kl_hessian}
\end{equation}

Thus, the KL-based merging objective retains the convexity
guarantees of the linearised formulation, but
does not admit the same Euclidean projection
interpretation or closed-form residual-energy
characterisation as the squared-output QP.
Its coefficients can instead be obtained using
convex optimisation methods.

\paragraph{Connection to Fisher Merging}
Note that $\operatorname{diag}(q_j)
-q_jq_j^\top$ is the covariance matrix of a categorical distribution with probabilities \(q_j\) which, under the softmax model, is the Fisher information with respect to the logits. Multiplication by \((A^{(j)})^\top\) and \(A^{(j)}\) transfers the Fisher information into the merging-coefficient space.
Thus, the Hessian can be interpreted as the Fisher information matrix with respect to the mask coefficients, evaluated at the current merged model under the affine-logit assumption. 

In contrast, Fisher merging \cite{matena_merging_2022} evaluates the Fisher information separately at each fine-tuned model and uses these as precision weights when combining parameters, while the method described here evaluates the Fisher matrix at the current merged model.

\subsection{Discussion of Loss Choice}
\label{sec:loss_discussion}

While squared error directly
penalises differences between output vectors
and provides a quadratic objective whose
optimal value can be characterised geometrically,
KL divergence instead matches predictive
distributions, making it particularly natural
for classification and language modelling,
where the relative probabilities assigned to
candidate outputs are important.

Both objectives are convex in the merging
coefficients under a fixed first-order
linearisation of the logits, although only
squared error yields a QP. Neither guarantees
optimal downstream task accuracy, since
calibration loss and task performance need
not coincide.

Other choices of loss, such as Wasserstein
distance, could incorporate additional geometry
between output categories, but need not
preserve convexity. We leave their investigation
to future work.

\section{Experiments}\label{exp}
We evaluate OSP merging against existing methods on both vision
and language benchmarks. A detailed description of the experimental
setup is provided in App.~\ref{exp_app}.

\textbf{Vision.} For the vision experiments, we begin with a CLIP-pretrained
ViT-B/32 \cite{dosovitskiy2020image} model and fine-tune separate task-specific models
on four datasets: CIFAR-10, STL-10, Imagenette, and EuroSAT.
We consider three fine-tuning and merging configurations, training only the parameters that are merged (on top of the base model with a frozen per-task classification head). We train and merge the transformer block MLP output projection(s) of: the final block, the final four blocks, and all twelve blocks. This allows us to investigate whether merging a
small number of later layers is sufficient to preserve
much of the task-specific improvements, and how performance
changes as we increase the number of layers included. We evaluate the merged
models using mean squared error (MSE), cross-entropy (CE),
KL divergence, and classification accuracy on each dataset,
as well as the aggregate performance across tasks.
\begin{table}[t]
\centering\scriptsize\setlength{\tabcolsep}{1pt}
\caption{Performance on ViT-B/32 across single-layer, 4-layer and 12-layer merging.}
\label{tab:vit_main}
\definecolor{ospkl}{HTML}{009E73}\definecolor{ospmse}{HTML}{D55E00}\definecolor{ospmselin}{HTML}{0072B2}
\definecolor{basegrey}{HTML}{7A8894}\definecolor{refgrey}{HTML}{C9CFD4}
\ifdefined\avgbarw\else\newlength{\avgbarw}\fi
\providecommand{\avgbar}[4]{\rlap{\color{black!12}\rule[-0.25ex]{\avgbarw}{1.5ex}}%
\rlap{\color{#4}\rule[-0.25ex]{#1\avgbarw}{1.5ex}}%
\rlap{\hspace{#2\avgbarw}\rule[0.45ex]{#3\avgbarw}{0.4pt}}\hspace{\avgbarw}}
\setlength{\avgbarw}{0.44cm}
\setlength{\aboverulesep}{0pt}\setlength{\belowrulesep}{0pt}\renewcommand{\arraystretch}{1.15}
\begin{tabular*}{\columnwidth}{@{\extracolsep{\fill}}c@{\hspace{3pt}}lcccc@{\extracolsep{0pt}\hspace{3\tabcolsep}}>{\columncolor[HTML]{F4F6F7}}c>{\columncolor[HTML]{F4F6F7}}l@{\hspace{4.5pt}}c}
\toprule
\rule{0pt}{1.25\normalbaselineskip}\rule[-0.55\normalbaselineskip]{0pt}{0pt} & Method & CIFAR-10 & STL-10 & Imagenette & EuroSAT & Avg. & \makebox[\avgbarw][c]{\tiny 50--100} & MSE \\
\midrule
& \rule{0pt}{1.25\normalbaselineskip}\rule[-0.55\normalbaselineskip]{0pt}{0pt}\textit{Fine-tuned} & \textit{91.7} & \textit{97.6} & \textit{98.8} & \textit{92.5} & \textit{95.2} & \avgbar{0.9031}{0.9028}{0.0006}{refgrey} & -- \\
\cmidrule{2-9}
& \rule{0pt}{1.25\normalbaselineskip}OSP-KL & \textbf{83.0} & \textbf{95.3} & 96.5 & \underline{84.0} & \textbf{89.7}\,{\tiny$\pm$0.2} & \avgbar{0.7940}{0.7891}{0.0098}{ospkl} & 1.90 \\
& OSP-MSE & 73.3 & 94.8 & 96.4 & 82.7 & 86.8\,{\tiny$\pm$0.4} & \avgbar{0.7361}{0.7287}{0.0149}{ospmse} & \underline{1.72} \\
& OSP-MSE (lin.) & 70.8 & \underline{95.0} & 94.8 & 67.6 & 82.0\,{\tiny$\pm$0.5} & \avgbar{0.6405}{0.6314}{0.0182}{ospmselin} & 2.46 \\
& RegMean & \underline{80.6} & 94.0 & \underline{96.7} & \textbf{84.6} & \underline{89.0}\,{\tiny$\pm$0.3} & \avgbar{0.7799}{0.7736}{0.0126}{basegrey} & \textbf{0.87} \\
& Fisher & 63.1 & 94.3 & \textbf{97.6} & 66.0 & 80.3\,{\tiny$\pm$1.1} & \avgbar{0.6052}{0.5828}{0.0448}{basegrey} & 2.69 \\
& \rule[-0.55\normalbaselineskip]{0pt}{0pt}Soup & 74.2 & 93.3 & 92.1 & 76.0 & 83.9\,{\tiny$\pm$0.4} & \avgbar{0.6782}{0.6706}{0.0152}{basegrey} & 2.71 \\
\cmidrule{2-9}
\multirow{-8}{*}[6.3pt]{\rotatebox[origin=c]{90}{\textbf{1 layer}}} & \rule{0pt}{1.25\normalbaselineskip}\rule[-0.55\normalbaselineskip]{0pt}{0pt}\textit{Base} & \textit{49.8} & \textit{63.2} & \textit{79.9} & \textit{53.3} & \textit{61.5} & \avgbar{0.2307}{0.2307}{0.0000}{refgrey} & \textit{4.94} \\
\midrule
& \rule{0pt}{1.25\normalbaselineskip}\rule[-0.55\normalbaselineskip]{0pt}{0pt}\textit{Fine-tuned} & \textit{96.3} & \textit{98.7} & \textit{99.5} & \textit{98.4} & \textit{98.2} & \avgbar{0.9648}{0.9641}{0.0013}{refgrey} & -- \\
\cmidrule{2-9}
& \rule{0pt}{1.25\normalbaselineskip}OSP-KL & \underline{91.4} & \textbf{96.6} & \underline{98.0} & 93.5 & \textbf{94.9}\,{\tiny$\pm$0.2} & \avgbar{0.8974}{0.8944}{0.0060}{ospkl} & 2.69 \\
& OSP-MSE & 88.6 & \underline{94.8} & \underline{98.0} & \textbf{95.6} & 94.2\,{\tiny$\pm$0.6} & \avgbar{0.8849}{0.8727}{0.0244}{ospmse} & \underline{1.85} \\
& OSP-MSE (lin.) & 82.9 & 94.4 & 91.1 & 73.9 & 85.6\,{\tiny$\pm$2.6} & \avgbar{0.7118}{0.6601}{0.1034}{ospmselin} & 4.05 \\
& RegMean & \textbf{92.4} & 92.8 & \textbf{98.3} & \underline{94.5} & \underline{94.5}\,{\tiny$\pm$0.6} & \avgbar{0.8897}{0.8772}{0.0250}{basegrey} & \textbf{1.22} \\
& Fisher & 83.0 & 91.3 & 93.6 & 69.5 & 84.3\,{\tiny$\pm$1.6} & \avgbar{0.6869}{0.6549}{0.0640}{basegrey} & 4.20 \\
& \rule[-0.55\normalbaselineskip]{0pt}{0pt}Soup & 86.7 & 89.0 & 81.4 & 68.0 & 81.3\,{\tiny$\pm$0.9} & \avgbar{0.6252}{0.6078}{0.0349}{basegrey} & 4.97 \\
\cmidrule{2-9}
\multirow{-8}{*}[6.3pt]{\rotatebox[origin=c]{90}{\textbf{4 layers}}} & \rule{0pt}{1.25\normalbaselineskip}\rule[-0.55\normalbaselineskip]{0pt}{0pt}\textit{Base} & \textit{49.8} & \textit{63.2} & \textit{79.9} & \textit{53.3} & \textit{61.5} & \avgbar{0.2307}{0.2307}{0.0000}{refgrey} & \textit{8.64} \\
\midrule
& \rule{0pt}{1.25\normalbaselineskip}\rule[-0.55\normalbaselineskip]{0pt}{0pt}\textit{Fine-tuned} & \textit{97.6} & \textit{98.0} & \textit{99.3} & \textit{98.7} & \textit{98.4} & \avgbar{0.9678}{0.9666}{0.0023}{refgrey} & -- \\
\cmidrule{2-9}
& \rule{0pt}{1.25\normalbaselineskip}OSP-KL & 90.7 & \textbf{94.8} & \textbf{96.8} & 94.8 & \textbf{94.2}\,{\tiny$\pm$0.7} & \avgbar{0.8848}{0.8701}{0.0295}{ospkl} & 2.81 \\
& OSP-MSE & 88.8 & \underline{88.4} & \underline{94.2} & \underline{95.7} & \underline{91.8}\,{\tiny$\pm$1.3} & \avgbar{0.8357}{0.8101}{0.0512}{ospmse} & \underline{2.00} \\
& OSP-MSE (lin.) & 86.4 & 85.8 & 80.2 & 67.0 & 79.8\,{\tiny$\pm$2.5} & \avgbar{0.5969}{0.5478}{0.0982}{ospmselin} & 4.84 \\
& RegMean & \textbf{95.8} & 73.6 & 92.6 & \textbf{97.0} & 89.7\,{\tiny$\pm$0.8} & \avgbar{0.7948}{0.7791}{0.0314}{basegrey} & \textbf{1.58} \\
& Fisher & 81.8 & 86.1 & 85.4 & 20.5 & 68.5\,{\tiny$\pm$3.6} & \avgbar{0.3691}{0.2975}{0.1432}{basegrey} & 5.16 \\
& \rule[-0.55\normalbaselineskip]{0pt}{0pt}Soup & \underline{91.5} & 77.9 & 71.6 & 73.9 & 78.7\,{\tiny$\pm$0.9} & \avgbar{0.5744}{0.5559}{0.0371}{basegrey} & 5.28 \\
\cmidrule{2-9}
\multirow{-8}{*}[6.3pt]{\rotatebox[origin=c]{90}{\textbf{12 layers}}} & \rule{0pt}{1.25\normalbaselineskip}\rule[-0.55\normalbaselineskip]{0pt}{0pt}\textit{Base} & \textit{49.8} & \textit{63.2} & \textit{79.9} & \textit{53.3} & \textit{61.5} & \avgbar{0.2307}{0.2307}{0.0000}{refgrey} & \textit{8.75} \\
\bottomrule
\end{tabular*}
\end{table}

We compare two approaches to optimising the merging
coefficients. The first directly minimises the squared-output
loss using 150 gradient steps (OSP by GD), without linearising
the network. The second uses a first-order approximation
of the downstream mapping after the merge layer to construct a convex quadratic
programme (OSP with linearised MSE). The latter admits
a globally optimal solution for the linearised objective
under convex coefficient constraints, but may incur
approximation error when evaluated on the non-linear model.
We use the standard diagonal basis throughout and additionally
test KL divergence as an alternative calibration loss.
All merging objectives are constructed using 200 calibration
examples per task. For multi-layer merging, both approaches optimise the coefficient of all merged layers jointly. We report results as means across 6 random seeded runs, unless otherwise specified.

\begin{table}[t]
\centering\scriptsize\setlength{\tabcolsep}{1.4pt}
\caption{Accuracy results from the merging of eight public LoRA adapters.}
\label{tab:public_lora_k8}
\definecolor{ospkl}{HTML}{009E73}\definecolor{ospmse}{HTML}{D55E00}\definecolor{ospmselin}{HTML}{0072B2}
\definecolor{basegrey}{HTML}{7A8894}\definecolor{refgrey}{HTML}{C9CFD4}
\ifdefined\avgbarw\else\newlength{\avgbarw}\fi
\providecommand{\avgbar}[4]{\rlap{\color{black!12}\rule[-0.25ex]{\avgbarw}{1.5ex}}%
\rlap{\color{#4}\rule[-0.25ex]{#1\avgbarw}{1.5ex}}%
\rlap{\hspace{#2\avgbarw}\rule[0.45ex]{#3\avgbarw}{0.4pt}}\hspace{\avgbarw}}
\setlength{\avgbarw}{0.8cm}
\setlength{\aboverulesep}{0pt}\setlength{\belowrulesep}{0pt}\renewcommand{\arraystretch}{1.15}
\begin{tabular*}{\columnwidth}{@{\extracolsep{\fill}}lcccccccc@{\extracolsep{0pt}\hspace{3\tabcolsep}}>{\columncolor[HTML]{F4F6F7}}c>{\columncolor[HTML]{F4F6F7}}l}
\toprule
\rule{0pt}{1.25\normalbaselineskip}\rule[-0.55\normalbaselineskip]{0pt}{0pt}Method & MNLI & WC & ATO & SNLI & Wino & R09 & MWOZ & TwE & Avg. & \makebox[\avgbarw][c]{\tiny 0--100\%} \\
\midrule
\rule{0pt}{1.25\normalbaselineskip}\rule[-0.55\normalbaselineskip]{0pt}{0pt}\textit{Fine-tuned} & \textit{92.7} & \textit{98.5} & \textit{96.1} & \textit{86.6} & \textit{99.7} & \textit{93.4} & \textit{100.0} & \textit{90.7} & \textit{94.7} & \avgbar{0.9470}{0.9469}{0.0002}{refgrey} \\
\midrule
\rule{0pt}{1.25\normalbaselineskip}OSP-KL & 88.5 & \underline{92.2} & 94.5 & 85.9 & 96.2 & 91.2 & \underline{99.5} & 74.3 & 90.3\,{\tiny$\pm$0.6} & \avgbar{0.9029}{0.8970}{0.0119}{ospkl} \\
OSP-MSE & \underline{89.6} & \textbf{94.6} & \textbf{95.6} & \underline{87.6} & \underline{99.2} & \underline{92.3} & \textbf{100.0} & \textbf{80.6} & \textbf{92.4}\,{\tiny$\pm$0.2} & \avgbar{0.9243}{0.9227}{0.0033}{ospmse} \\
RegMean & \textbf{92.9} & 86.4 & \underline{94.8} & \textbf{88.4} & \textbf{99.6} & \textbf{92.7} & 98.2 & \underline{75.0} & \underline{91.0}\,{\tiny$\pm$0.2} & \avgbar{0.9101}{0.9083}{0.0035}{basegrey} \\
Fisher & 39.8 & 66.9 & 78.7 & 50.4 & 62.9 & 73.8 & 53.1 & 62.1 & 61.0\,{\tiny$\pm$1.9} & \avgbar{0.6095}{0.5902}{0.0387}{basegrey} \\
\rule[-0.55\normalbaselineskip]{0pt}{0pt}Soup & 33.8 & 68.9 & 55.3 & 66.2 & 52.8 & 75.5 & 47.2 & 52.4 & 56.5\,{\tiny$\pm$0.0} & \avgbar{0.5651}{0.5651}{0.0001}{basegrey} \\
\midrule
\rule{0pt}{1.25\normalbaselineskip}\rule[-0.55\normalbaselineskip]{0pt}{0pt}\textit{Base} & \textit{17.7} & \textit{37.1} & \textit{45.2} & \textit{40.2} & \textit{38.1} & \textit{45.1} & \textit{44.8} & \textit{47.8} & \textit{39.5} & \avgbar{0.3948}{0.3942}{0.0012}{refgrey} \\
\bottomrule
\end{tabular*}
\end{table}

Table~\ref{tab:vit_main} shows that OSP with KL divergence optimised by gradient descent achieves the highest average accuracy in all three ViT configurations, with OSP-MSE by GD and RegMean close behind. Results across further benchmark methods are presented in App.~\ref{vit}. The linearised QP is competitive when few layers are merged but degrades when all twelve blocks are merged, due to accumulating linearisation errors.


\textbf{Language.} We consider a range of common settings for merging language model fine-tunes. In this section, we present results for the merging of publicly available LoRA adapters, whilst App.~\ref{llm} and App.~\ref{llm-lora} explore the merge of a full finetune of an LLM and of self-trained LoRA adapters.

We take publicly available LoRA adapters from the `Lots-of-LoRAs' hub \citep{GabrielssonZBCG25} - these are rank-16 adapters on the Mistral-7B-Instruct model, trained on tasks from the Super-Natural Instructions dataset \citep{wang-etal-2022-super}. We select a subset of tasks where the expert LoRA model has substantial gain over the unadapted base model, for closed-label classification tasks. Table \ref{tab:public_lora_k8} shows that OSP solved with MSE by gradient descent outperforms all other baselines here; results for all baselines and further task sets are given in App.~\ref{llm-lora-public}.

We also explore training our own LoRA adapters on a bigger model and finetuning an LLM in the same manner as we did the vision transformer. In both of these cases, we find that OSP consistently rivals or outperforms the best competitor methods. These results are shown in App.~\ref{llm} and App.~\ref{llm-lora}.

\newpage
\section*{Acknowledgements}
The authors acknowledge support from His Majesty’s Government in the development of this research.
Jared Tanner is supported by the UK Engineering and Physical Sciences Research Council (EPSRC) through the grant EP/Y028872/1.
\bibliographystyle{plainnat}
\bibliography{references}

\appendix
\onecolumn
\section{Related Work}
\label{sec:related}

\paragraph{Task arithmetic and model soups.}
\citet{ilharco2023editing} introduced task vectors and showed that arithmetic
operations on them transfer to model behaviour.  \citet{wortsman2022model} showed
that averaging fine-tuned weights improves robustness.  Both are recovered as
special cases of our framework.

\paragraph{TIES and DARE.}
\citet{yadav2023ties} introduced sign-based conflict resolution; \citet{yu_language_2024}
introduced random dropping of residual parameters.  We show  that DARE with row-uniform masks is a special case of our framework.

\paragraph{Fisher-weighted merging.}
\citet{matena_merging_2022} weight parameters by their Fisher information.  Our method is conceptually related to precision-weighted merging (Fisher), but operates in the standard basis of residual updates rather than parameter curvature. Fisher merging is a related parameter-space quadratic objective using Fisher precision, rather than an exact instance of our squared-output QP.

\paragraph{RegMean.}
RegMean~\citep{jin2023regmean} computes merge weights by solving a closed-form linear regression at each layer, using second-moment statistics of activations to match the layer-wise outputs of the fine-tuned models. It is a related layer-wise least-squares approach, but operates on parameter-space activation statistics rather than the output-space residual objective optimised here, and we do not claim it as a special case of our framework.

\paragraph{ZipIt!.}
ZipIt!~\citep{stoica_zipit_2024} merges models by identifying and aligning correlated features across networks, effectively constructing a shared representation before combining parameters. In our QP framework, this can be interpreted as a basis alignment step that transforms residual updates into a space where the corresponding directions are more compatible. Unlike diagonal masking, which fixes the basis and optimises coefficients, ZipIt! modifies the representation itself to reduce misalignment between models. Thus, it can be viewed as complementary to our approach, in that ZipIt! seeks a favourable basis in which merging is easier, while our method optimises merging within a given basis.

\paragraph{Subspace and spectral methods.}
KnOTS aligns LoRA updates by stacking the low-rank update matrices then performing an SVD to obtain a shared orthonormal basis. Then expressing each update in that basis and merging coefficients in the aligned space. More broadly, this fits a family of low-rank and core-space alignment methods that choose a basis from the residual updates themselves. In our framework these correspond to choosing $\{q_p\}$ from the column space of stacked residuals rather than from the eigenbasis of $S$.
\begin{table}[h]
\centering
\small
\begin{tabular}{p{2.6cm} p{3.0cm} p{3.9cm} p{5.0cm}}
\toprule
\textbf{Method} & \textbf{Parametrisation} & \textbf{Coefficients} & \textbf{Relation to OSP} \\
\midrule
\multicolumn{4}{l}{\emph{Special cases of the diagonal parametrisation (feasible points when the coefficients lie in $\mathcal C$)}} \\
\addlinespace
Task arithmetic & Standard basis & $d_{N,k,i}=\lambda_k$ & Fixed feasible point; in $[0,1]$ if $\lambda_k\le 1$ \\
Model soup & Standard basis & $d_{N,k,i}=1/K$ & Fixed feasible point \\
AdaMerging (layer-wise) & Standard basis & $d_{N,k,i}=\lambda_k^{(N)}$, learned by entropy minimisation & Feasible point; constant across rows, different objective \\
DARE (row-wise) & Standard basis & $d_{N,k,i}=m_i^{(k)}/\rho$, $m_i^{(k)}\sim\mathrm{Bernoulli}(\rho)$ & Random point; in $\mathcal C$ if $1/\rho$ is admissible. Element-wise DARE is not captured \\
\midrule
\multicolumn{4}{l}{\emph{OSP (ours)}} \\
\addlinespace
Diagonal QP & Standard basis & Chosen to minimise linearised squared-output loss & Global optimum of the linearised calibration objective within the diagonal family (Sec.~\ref{diagonal_qp}) \\
General-basis QP & Orthonormal $Q_N$ & Chosen to minimise linearised squared-output loss & Global optimum within the chosen basis (Prop.~\ref{thm:general-qp}) \\
OSP with KL Linearised & Standard basis & Chosen to minimise linearised KL divergence & Convex (Sec.~\ref{sec:kl_loss}); global optimum of the linearised objective, but not a QP \\
OSP by GD (MSE or KL) & Standard basis & Chosen by gradient descent on the non-linearised loss & No linearisation error, but non-convex; no global-optimality guarantee \\
Eigenbasis of $S$ & Top-$p$ eigenvectors of $S$ & -- & Optimal output subspace for the relaxed problem (Cor.~\ref{cor}) \\
\midrule
\multicolumn{4}{l}{\emph{Related methods outside the diagonal parametrisation}} \\
\addlinespace
TIES & Element-wise & Trim, sign election, disjoint mean & Sparsity and sign constraints; element-wise operations not captured by row masks \\
SVD-based (e.g.\ KnOTS) & Singular vectors of $\delta_N^{(k)}$ & Method-specific & An alternative choice of basis $Q_N$; see App.~\ref{sec:svd} for a closed-form example \\
Fisher merging & Parameter-wise Fisher precision & Precision-weighted average & Quadratic objective in parameter space, not output space \\
RegMean & Full weight matrix & Closed-form layer-wise regression & Least squares on layer outputs; richer than diagonal scaling, not a special case \\
ZipIt! & Learned feature alignment & -- & Complementary: changes the representation before merging \\
Multi-teacher distillation \cite{dalili_model_2025} & Task- or layer-wise coefficients & KL to fine-tuned teachers, with sharpness-aware optimisation & Same KL objective as OSP-KL, but optimised with SAM over coarser coefficients\\
\bottomrule
\end{tabular}
\caption{Relationship of common model merging methods to OSP. The first group are
exact special cases of the diagonal parametrisation, so the diagonal QP attains
linearised calibration loss no greater than any of them whose coefficients lie in
the constraint set $\mathcal C$.}
\label{table}
\end{table}












\subsection{SVD-Basis}
\label{sec:svd}
Now let us consider setting the basis to the minimal set of singular vectors of the fine-tuned residual weights in the QP. 
We consider the singular value decomposition (SVD) of each of the fine-tuned residuals, and use this to derive optimal scaling weights applied to each singular direction. We begin by optimising for one direction only, then extend to multiple directions.

For each task $k$ let $\delta_N^{(k)}=\sum_{s=1}^{r_k}\sigma^{(k)}_s u^{(k)}_s(v^{(k)}_s)^\top$
be the SVD of the layer-$N$ update, and fix a unit vector $u\in\mathbb R^{d_N}$ as the
direction $q$.

\begin{assumption}[Shared direction]\label{ass:shared}
Suppose there exists a set $\mathcal T\subseteq\{1,\dots,K\}$ such that for every $k\in\mathcal T$,
$\delta_N^{(k)}=\sigma^{(k)}uv^\top+R^{(k)}$ with a common unit vector $v$ and
$u^\top R^{(k)}=0$, while $u^\top\delta_N^{(k)}=0$ for every $k\notin\mathcal T$.
\end{assumption}

\begin{assumption}[Unit output direction]\label{ass:iso}
$J_N(x)^\top J_N(x)\,u=u$.
\end{assumption}

Under Assumption~\ref{ass:shared}, $m_k=\sigma^{(k)}s$ for $k\in\mathcal T$ and $m_k=0$
otherwise, where $s:=v^\top z$ is a scalar common to all models.

\begin{theorem}[Signal-to-total-power merging weights]
Suppose Assumptions~\ref{ass:shared}--\ref{ass:iso} hold, $s\neq0$, and the target is
the linearised output of task $t\in\mathcal T$, i.e.\
$y=f(x;\theta_0)+J_N(x)\,\delta_N^{(t)}z$. Then $\beta=-\sigma^{(t)}s$, and the
minimum-norm minimiser of~\eqref{eq:1d-obj} is
\[
d^\ast_k=\frac{\sigma^{(t)}\sigma^{(k)}}{\sum_{l\in\mathcal T}(\sigma^{(l)})^2}
\quad (k\in\mathcal T),
\qquad d^\ast_k=0\quad(k\notin\mathcal T).
\]
\end{theorem}

\begin{corollary}[Recovery of existing methods in the SVD basis]
If $\sigma^{(k)}=\sigma$ for all $k\in\mathcal T$, then $d^\ast_k=1/|\mathcal T|$, recovering
model averaging. If $\mathcal T=\{t\}$ (i.e each model has unique and orthogonal singular vectors to eachother), then $d^\ast_t=1$, recovering direct addition.
\end{corollary}
\section{Additional Proofs}\label{proofs}

\subsection{Quadratic Programme}\label{app}  
For a single calibration sample $(x_j, y_j)$, the difference between the merged model outputs with the target output is
\begin{align}
f(x_j;\theta_0+\delta_N(\mathbf d))-y_j
&=
\bigl(f(x_j;\theta_0)-y_j\bigr)
+
\sum_{k=1}^K
J_N(x_j) D_{N,k} \delta_N^{(k)} z_j \\
&=
b_j
+
\sum_{k=1}^K
\left(\sum_{i=1}^{d_N} J_N(x_j)\, d_{N,k,i}\, e_i e_i^\top\right)\delta_N^{(k)} z_j \\
&=
b_j
+
\sum_{k=1}^K
\sum_{i=1}^{d_N}
d_{N,k,i}
\bigl(e_i^\top \delta_N^{(k)} z_j\bigr)
J_N(x_j) e_i .
\end{align}
Since
\(
e_i^\top \delta_N^{(k)}
=
(\delta_N^{(k)})_{i,:},
\)
the scalar term may be written in terms of the $i$th row of $\delta_N^{(k)}$ as
\[
e_i^\top \delta_N^{(k)} z_j
=
(\delta_N^{(k)})_{i,:}\, z_j,
\]
i.e. the $i$-th row of the fine-tuned update acting on the input vector into the $N$th layer $z_j$. This quantity is a scalar which measures the amplitude of the $i$th row of $\delta_N^{(k)}$ aligned with the input vector of sample $x_j$ into the $N$th layer.\\

\noindent
Let us write $J^{(j)}:=J_N(x_j)$, then 
similarly,
\(
J^{(j)} e_i = J^{(j)}_{:,i},
\)
gives the $i$-th column of the Jacobian of the downstream map $J^{(j)}$.
Hence
\begin{align}
f(x_j;\theta_0+\delta_N(\mathbf d))-y_j
=
b_j
+
\sum_{k=1}^K
\sum_{i=1}^{d_N}
\quad \underbrace{d_{N,k,i}}_{\substack{\text{merge}\\[1mm]\text{weight}}} \quad
\underbrace{\bigl((\delta_N^{(k)})_{i,:} z_j\bigr)}_{\substack{
\text{activation of the $i$-th row}\\[1mm]
\text{of the $k$-th fine-tuned update}\\[1mm]
\text{on input data $x_j$}
}}
\quad
\underbrace{J^{(j)}_{:,i}}_{\substack{\text{$i$-th column}\\[1mm]\text{of $J^{(j)}$}}}.
\end{align}
We highlight that both $d_{N,k,i}$ and $(\delta_N^{(k)})_{i,:} z_j$ are scalars which are applied to $J^{(j)}_{:,i}$. Thus the diagonal mask scales each residual row independently, while the
columns of $J^{(j)}$ determine the output-space directions along which corrections
may be applied. The merged correction is therefore a linear combination of the
fixed output directions $\{J^{(j)}_{:,i}\}_{i=1}^{d_N}$ (given by the columns of $J^{(j)}$) with sample-dependent amplitudes
determined by the update activations $r_{k,j}$ and weightings $d_{N,k,i}$, i.e. $A^{(k,j)} = J_N(x_j) diag(r_{k,j})$, as in Sec. \ref{diagonal_qp}. Thus, for each sample the induced output space realisable by left-diagonal masks applied to the fine-tuned updates lies in
\[
\mathcal S_{\mathrm{left}}
:=
\operatorname{span}\{ J^{(j)}e_1,\ldots,J^{(j)} e_{d_N}\} = \operatorname{col}(J^{(j)}).
\]

\noindent
Equivalently, for each sample $j$, we may write the desired correction in output space as:
\[ 
f(x_j;\theta_0+\delta_N(\mathbf d))-y_j
=
b_j + A^{(j)} \mathbf d ,
\]
where $A^{(j)}$ is the block matrix defined by 
\[
A^{(j)} : = \left[ A^{(1, j)}, \dots , A^{(K, j)} \right]
\]
whose $k$-th block $A^{(k, j)}$ has columns
\[ A^{(k, j)} := \left[
\bigl(e_1^\top \delta_N^{(k)} z_j\bigr)J^{(j)}e_1,\;
\ldots,\;
\bigl(e_{d_N}^\top \delta_N^{(k)} z_j\bigr)J^{(j)}e_{d_N} \right] 
\]
where $d_N$ is the number of rows of $\delta_N^{(k)}$, and $\mathbf d$ is the stacked weight vector defined by
\[
\mathbf d := \left[ d_{N,1,1} , \ldots , d_{N,1,d_N} , 
\ldots , d_{N,K,1} , \ldots , d_{N,K,d_N}\right]^\top
\]

\noindent
Therefore the calibration objective is a sum of squared norms over samples in the calibration set $\mathcal K$ of the form
\[
\cJ_I(\mathbf d)
=
\sum_{j\in\mathcal K}\|A^{(j)}\mathbf d + b_j\|_2^2,
\]
where $\|\cdot\|_2$ denotes the vector Euclidean norm. This is a convex quadratic function of the shared merge coefficients $\mathbf d$. This is equivalent to the following formulations where $n:= \vert \mathcal K \vert$:
\begin{itemize}
    \item For the stacked matrix $A$ and stacked offset vector $b$: \[
\cJ_I(\mathbf d)
=
\left\|
\begin{bmatrix}
A^{(1)} \\
A^{(2)}\\
\vdots \\
A^{(n)}
\end{bmatrix}\mathbf d
+
\begin{bmatrix}
b_{1} \\
b_{2} \\
\vdots \\
b_{n}
\end{bmatrix}
\right\|_2^2 := \left\|A\mathbf d + b\right\|_2^2.
\]

\item 
Equivalently, if we reshape the per-sample residuals into a matrix
\[
R(\mathbf d)
:=
\begin{bmatrix}
A^{(1)}\mathbf d + b_{1} &
A^{(2)}\mathbf d + b_{2} &
\cdots &
A^{(n)}\mathbf d + b_{n}
\end{bmatrix},
\]
then
\[
\cJ_I(\mathbf d)=\|R(\mathbf d)\|_F^2,
\]
where $\|\cdot\|_F$ denotes the Frobenius norm.
\end{itemize}

\subsection{Proof of Theorem~\ref{projection}}
\label{prof_proof}
\begin{proof}
Let $\mathcal S\subseteq\mathbb{R}^c$ be a subspace with orthogonal projector
$P_{\mathcal S}$. In the relaxed problem each $b_j$ is corrected independently
by some $c_j\in\mathcal S$, and $\min_{c_j\in\mathcal S}\|b_j+c_j\|^2$ is attained
at $c_j=-P_{\mathcal S}b_j$, leaving residual $(I-P_{\mathcal S})b_j$. Therefore
the relaxed squared-output loss over the calibration set $\mathcal K$ is
\begin{equation}
  \label{eq:error-proj}
  \cJ_{\mathcal S}
  \;=\;
  \sum_{j\in\cK} \big\|(I-P_{\mathcal S}) b_j\big\|^2
  \;=\;
  \sum_{j\in\cK} \big( \|b_j\|^2 - \|P_{\mathcal S} b_j\|^2 \big).
\end{equation}
Since \(\sum_j\|b_j\|^2\) is independent of the choice of \(\mathcal S\), the
problem of selecting the best output subspace \(\mathcal S\) of fixed
dimension reduces to maximising the captured energy
\(\sum_j\|P_{\mathcal S} b_j\|^2\).  Using the cyclicity of trace we can
rewrite the captured energy in terms of \(S\) and the projector:
\begin{equation}
  \label{eq:trace-form}
  \sum_{j\in\cK} \|P_{\mathcal S} b_j\|^2
  \;=\;
  \sum_{j\in\cK} \operatorname{tr}\big( b_j b_j^\top P_{\mathcal S}\big)
  \;=\;
  \operatorname{tr}\big(S\,P_{\mathcal S}\big).
\end{equation}
Hence we obtain that the minimal loss attainable on $\mathcal S$ is
\begin{equation}\nonumber
  \cJ_{\mathcal S} \;=\; \sum_{j\in\cK}\|b_j\|^2 \;-\; \operatorname{tr}(S\,P_{\mathcal S}),
\end{equation}
and the loss difference between two subspaces is:
\[
\cJ_{\mathcal S_2}-\cJ_{\mathcal S_1}
=\operatorname{tr}\big(S(P_{\mathcal S_1}-P_{\mathcal S_2})\big).
\]
\end{proof}

\section{Single Direction General Basis Case}

First we consider a single direction $q$.
Define the \emph{scalar projection}
of model $k$'s update activation onto $q$ for input $x$:
\begin{equation}
  m_k \;:=\; q^\top\delta_N^{(k)}h_{N-1}(x;\theta_0) \;\in\;\R,
\end{equation}
and the projection of the base-model residual onto the downstream direction $J_N(x)q$:
\begin{equation}
  \beta \;:=\; (J_N(x)q)^\top b, \qquad b \;:=\; f(x;\theta_0)-y.
\end{equation}
Merging along $q$ changes the output only along $J_N(x)q$, which we assume has unit
norm, $\|J_N(x)q\|=1$. Restricting the objective to the component of outputs along
$J_N(x)q$ (the orthogonal component is unaffected by $\mathbf d$), the 1-D
objective for a single data point $(x,y)$ is
\begin{equation}
  \label{eq:1d-obj}
  \cJ(\mathbf d) \;=\; \Bigl(\sum_{k=1}^K d_k\,m_k + \beta\Bigr)^2
  \;=\; \bigl(\mathbf d^\top \mathbf m + \beta\bigr)^2,
\end{equation}
with $\mathbf d=(d_1,\dots,d_K)^\top$ and $\mathbf m=(m_1,\dots,m_K)^\top\in\R^K$.

\begin{proposition}[1-D optimal mask]
\label{prop:1d}
If $\mathbf m \neq 0$, the minimum-norm minimiser of \eqref{eq:1d-obj}
over $\mathbf d\in\R^K$ is
\begin{equation}
  \label{eq:1d-sol}
  \mathbf d^\ast \;=\; -\frac{\beta\, \mathbf m}{\norm{\mathbf m}^2}.
\end{equation}
If $\mathbf m = 0$, the objective is constant and every $\mathbf d$ is optimal.
\end{proposition}
Without the normalisation, the same argument gives
$\mathbf d^\ast=-\beta\,\mathbf m/(\|J_N(x)q\|^2\norm{\mathbf m}^2)$.
Thus we see that the diagonal mask value is given by an energy weighting of the fine-tuning in this direction on a given input as well as the back-propagated energy in the output from that direction.

\section{Diagonal Masking from the Right}\label{right}

In Section~\eqref{diagonal_qp}, we apply the diagonal mask from the left rather than the right. This is a modelling choice; a dual right-masked formulation can also be derived, although it induces a different parametrisation of the update.

For right masking we write $\tilde D_{N,k}=\operatorname{diag}(\tilde{\mathbf d}_{N,k})
\in\mathbb R^{d_{N-1}\times d_{N-1}}$, with coefficients $\tilde d_{N,k,i}$ scaling the
columns of $\delta_N^{(k)}$.
For a single calibration sample $(x_j, y_j)$, if the diagonal mask is applied from the right, the difference between the merged model output and the target output is
\begin{align}
f(x_j;\theta_0+\delta_N(\mathbf{\tilde{d}}))-y_j
&=
\bigl(f(x_j;\theta_0)-y_j\bigr)
+
\sum_{k=1}^K
J_N(x_j) \delta_N^{(k)} \tilde{D}_{N,k} z_j \\
&=
b_j
+
\sum_{k=1}^K
J_N(x_j) \delta_N^{(k)}
\left(\sum_{i=1}^{ d_{N-1}} \tilde{d}_{N,k,i}\, e_i e_i^\top\right) z_j \\
&=
b_j
+
\sum_{k=1}^K
\sum_{i=1}^{d_{N-1}}
\tilde{d}_{N,k,i}
\bigl(e_i^\top z_j\bigr)\,
J_N(x_j) \delta_N^{(k)} e_i .
\end{align}
Since  \(
e_i^\top z_j = (z_j)_i,
\)
the scalar term may be written in terms of the $i$-th coordinate of the input to the $N$th layer as
\[
e_i^\top z_j = (z_j)_i,
\]
i.e.\ the $i$-th coordinate of the hidden activation vector entering the fine-tuned layer for sample $x_j$. This quantity is a scalar measuring the amplitude of the $i$-th input feature into layer $N$ on sample $x_j$.\\

\noindent
Similarly,
\[
\delta_N^{(k)} e_i = (\delta_N^{(k)})_{:,i},
\]
gives the $i$-th column of the fine-tuned update $\delta_N^{(k)}$. Hence
\begin{align}
f(x_j;\theta_0+\delta_N(\mathbf{\tilde d} ))-y_j
=
b_j
+
\sum_{k=1}^K
\sum_{i=1}^{d_{N-1}}
\quad \underbrace{\tilde d_{N,k,i}}_{\substack{\text{merge}\\[1mm]\text{weight}}} \quad
\underbrace{\bigl((z_j)_i\bigr)}_{\substack{
\text{$i$-th coordinate of the }\\[1mm]
\text{input feature to the $N$th layer}
}}
\quad
\underbrace{J_N(x_j)(\delta_N^{(k)})_{:,i}}_{\substack{\text{$i$-th column}\\[1mm]\text{of }J_N(x_j)\delta_N^{(k)}}}.
\end{align}
Under the same linearisation used above, the calibration residual remains affine in the shared merge coefficients, so the resulting least-squares objective is again a convex quadratic programme in $\mathbf{\tilde{d}}$. The left and right masked parametrisations therefore differ not in convexity, but in interpretation: left masking weights residual rows and aligns with an output-space projection view, whereas right masking weights residual columns and emphasises the input-feature side of the update. In practice then, the preferred choice depends on which structure is more natural for the task or architecture at hand.

The output-space interpretation of right-masking is less direct than in the left-masking case because the induced output directions depend on the columns of the fine-tuned updates. The induced output space realisable by right-diagonal masks applied to the fine-tuned updates lies in
\begin{align}
\mathcal S_{\mathrm{right}}
:&= \operatorname{span}\bigl\{J_N(x_j)\delta_N^{(k)}e_1, \ldots , J_N(x_j)\delta_N^{(k)}e_{d_{N-1}}
:\; k=1,\ldots,K\bigr\} \\
&=
\operatorname{span}\bigl\{J_N(x_j)(\delta_N^{(k)})_{:,i}
:\; k=1,\ldots,K,\; i=1,\ldots,d_{N-1}\bigr\}  .
\end{align}
Note that $\mathcal S_{\mathrm{left}}$ does not depend on the fine-tuned models, only on the downstream map and the chosen basis (or only the columns of the downstream map in the standard basis case), whereas  $\mathcal S_{\mathrm{right}}$ depends on the columns of each of the fine-tuned updates, the downstream map and the chosen basis.

\section{Experimental Method}\label{exp_app}

\paragraph{Setup:} We implement our merging experiments using Python 3.13 and PyTorch 2.11 (CUDA 12.6). We run our experiments on a machine equipped with four NVIDIA H100 NVL GPUs, an Intel Xeon Gold 5418Y 2.00GHz 24-core processor, 1.5TiB of RAM.

\subsection{Common Merging Protocol}

\paragraph{Setting} There are $K$ experts, each fine-tuned from a common pre-trained model on one task. A merge combines their weight updates $\delta_\ell^{(k)} = W_\ell^{(k)} - W_\ell$ at a fixed set of merge sites $\ell$; all other parameters retain their pre-trained values. As is standard practice in the literature, each task keeps its own classification head and heads are not merged (we intentionally under-train these heads such that task performance is concentrated in the merge sites).

\paragraph{OSP parametrisation} We use the diagonal (per-row) basis with one coefficient $d_{\ell,k,i}$ per (layer, task, output row). For LoRA, $\delta_\ell^{(k)} = s\, B_\ell^{(k)} A_\ell^{(k)}$ with $s=\frac{\alpha}{r}$ taken from the LoRA adapter's configuration.

\paragraph{OSP optimisation procedure} All $\ell \in \mathcal{N}$ are optimised jointly using Adam, at a learning rate of $10^{-2}$, initialised at model soup parameters, and projected onto $\mathcal{C}$ (clamped to $[0,1]$) after each step, for 150 steps. We use 200 calibration samples per task (with a 25\% held-out set) and report accuracy of the merge with the lowest held-out loss.

\paragraph{Baselines} All baselines report results for the same experts, calibration dataset, and evaluation process.

We evaluate OSP against a range of competitive baselines (Table~\ref{tab:baselines}), calibrating hyperparameters on the calibration set. Only the selected configuration is then evaluated on the test set. For full implementation details, we refer the reader to the respective papers.

\begin{table}[h!]
\centering
\caption{Merging baselines and respective calibration sweeps.}
\label{tab:baselines}
\small
\setlength{\tabcolsep}{4pt}
\renewcommand{\arraystretch}{1.15}
\begin{tabularx}{\columnwidth}{@{}l X@{}}
\toprule
Baseline & Calibration sweep \\
\midrule
Model Soup \citep{wortsman2022model} & none (coefficients $\frac{1}{K}$) \\
Task arithmetic \citep{ilharco2023editing} & $\lambda \in \{0.05, 0.10, \dots, 1.0\}$ \\
TIES \citep{yadav2023ties} & density $\in \{0.1, 0.2, 0.3, 0.5, 0.7, 1.0\}$ $\times$ $\lambda \in \{0.1, 0.2, \dots, 2.0\}$ \\
DARE \citep{yu_language_2024} & $p \in \{0.1, 0.3, 0.5, 0.7, 0.9, 0.99\}$ $\times$ $\lambda \in \{0.05, 0.10, \dots, 1.0\}$ \\
RegMean \citep{jin2023regmean} & $\alpha \in \{0.1, 0.3, 0.5, 0.7, 0.9, 1.0\}$ \\
Fisher \citep{matena_merging_2022} & none ($n = 200$ per task; 32 on public LoRA) \\
AdaMerging \citep{yang_adamerging_2024} & none (init.\ at tuned task-arithmetic $\lambda$; Adam, lr $10^{-3}$, 500 steps) \\
KnOTS-TIES \citep{stoica_model_2024} & as TIES (LoRA only) \\
\bottomrule
\end{tabularx}
\end{table}

As references, we also evaluate the fine-tuned experts and base model (with task head) on each task.

\paragraph{Evaluation}

We run our experiments across six seeds per configuration, reporting top-1 accuracy as an unweighted mean over the $K$ merged tasks. Where we train our own finetunes, the training process is also repeated with the different seeds. In figures, we present one dot per seed; in tables we present the mean with $\pm 95\%$ $t$-interval.

\subsection{Vision}

We perform merging on the CLIP ViT-B/32 \citep{radford2021learning, dosovitskiy2020image} with OpenAI's public pre-trained weights. This model has 12 transformer blocks; we perform fine-tuning and merging on the second MLP projection (`fc2') of the last 1, 4, or all 12 blocks. Only these layers are fine-tuned - otherwise, experts do not differ from the base model.

\paragraph{Data} We study merging across $K=4$ tasks, each of 10 classes. These are:
\begin{itemize}
\item{CIFAR-10} \cite{krizhevsky_learning_nodate} - $3 \times 32 \times 32$ dimension colour images drawn uniformly from ten object categories.

\item{STL-10}  \cite{coates_analysis_nodate} - $3 \times 96 \times 96$ dimension colour images drawn from 10 categories.

\item{Imagenette} - a public 10-class subset of ImageNet \cite{deng_imagenet_2009}, dimension $3\times 160 \times 160$.  

\item{EuroSAT} \cite{helber_eurosat_2017} - land use and land cover classification images based on Sentinel-2 satellite imagery. Each patch measures $3 \times 64 \times 64$ (we use the RGB variant, rather than multispectral).

\end{itemize}

\paragraph{Training} We first train one lightweight linear head per task, with the frozen pre-trained ViT backbone. We train this for $100$ AdamW steps (learning rate $5 \times 10^{-5}$, batch size $64$, weight decay $10^{-4}$) and then freeze the parameters. This is our `base' model upon which we train `expert' finetunes - it is intentionally under-trained to allow headroom for expert training. After this, we train experts on task-specific data for 15 epochs, using AdamW (learning rate $5 \times 10^{-5}$, batch size $64$, cosine schedule). Results for the vision experiments are shown in Appendix~\ref{vit}.

\subsection{Language}

We consider the merging of transformer-based language models across various contexts. First in models we train ourselves -- we assess this both in a fine-tune (akin to the above ViT setup) and in training LoRA adapters. Second, in the merging of a selection of publicly available LoRA adapters.

\paragraph{Data}

For LLM finetuning, we use a subset of $K=4$ classification tasks from the GLUE dataset \citep{wang2018glue}:
\begin{itemize}
    \item{SST-2} - determine movie review sentiment (positive / negative)
    \item{MNLI} - determine whether premise supports or contradicts a hypothesis (entailment / contradiction / neutral)
    \item{QNLI} - determine whether a candidate sentence contains the answer to a question (entailment / not entailment)
    \item{QQP} - determine whether two questions ask the same thing (duplicate / not duplicate) 
\end{itemize}

We train on a 8000 sample draw of each category. For testing, we use the GLUE validation set (because the public GLUE test set does not provide ground-truth labels).

\paragraph{Fine-tune Training - Qwen3-4B-Base}

This fine-tune is an equivalent to the process used in the vision transformer. First we do a lightweight train of classification heads on top of the frozen base (100 AdamW steps, learning rate $5 \times 10^{-5}$,  weight decay $10^{-4}$, batch size $64$) and then freeze these.

Then, we fine-tune 3 scenarios - the multilayer perceptron down projection block of the last 1, 4, or 36 (all) layers, with other parameters frozen. This is trained for 3 epochs over the 8000 samples, using AdamW (learning rate $10^{-4}$, weight decay $10^{-4}$, cosine schedule). The backbone is held at BF16 precision while the trainable parameters are FP32 (giving an FP32 residual stream for merging). Results are shown in Appendix~\ref{llm}.

\subparagraph{LoRA Training - Qwen3-14B-Base}

First, we train the model's classification heads in an identical process to the above. Next, we train rank 16 LoRA adapters on every linear map of every decoder layer (attention q/k/v/o, MLP gate/up/down). These are trained for 1 epoch using AdamW with learning rate $5 \times 10^{-4}$, batch size 32, and a cosine scheduler. The merge sites are the 280 adapted low rank matrices, with $\delta_\ell^{(k)} = s\,B_\ell^{(k)}A_\ell^{(k)}$. Results are shown in Appendix~\ref{llm-lora}.

\paragraph{Public LoRAs}

A motivating scenario for model merging is to be able to combine various publicly available model fine-tunes. For language models, these are commonly publicly provided as Low Rank Adapters for large open-weight models, on libraries such as HuggingFace. To assess model merging in this setting, we run experiments on Mistral-7B-Instruct \citep{mistral2023mistral7b} with LoRA adapters selected from the Lots of LoRAs library \citep{GabrielssonZBCG25}, which use tasks from the Super-NaturalInstructions benchmark \citep{wang-etal-2022-super}. We greedily select\footnote{With the exception of Twitter emotion; initially we greedily selected the SemEval humour-edit task (task496), but found RegMean's performance to be exceptionally poor on this. To provide a competitive comparison, we selected the next greedy task, Twitter emotion (task512). Results on the original set are given in Appendix~\ref{llm-lora-public}.} $K=8$ tasks which have high adapter accuracy gain over the base model before any merge. These are:

\begin{itemize}
    \item \textbf{MNLI neutral} (\texttt{task201}): given a premise and three candidate hypotheses, pick the one that is neutral with respect to the premise (labels 1 / 2 / 3).
    \item \textbf{Word count} (\texttt{task158}): count how many times a given word occurs in a sentence (labels 1--4).
    \item \textbf{ATOMIC oWant} (\texttt{task1198}): given a head and a tail event, decide whether, as a result of the head, others want the tail (Yes / No).
    \item \textbf{SNLI} (\texttt{task190}): decide whether a sentence pair is contradictory or neutral (C / N).
    \item \textbf{WinoMT gender} (\texttt{task351}): decide whether the gender of a named profession is identifiable from an anti-stereotypical sentence (Identifiable / Unidentifiable).
    \item \textbf{ROOT09 relations} (\texttt{task1505}): classify the semantic relation between two words (coordinate / hypernym / random).
    \item \textbf{MultiWOZ speaker} (\texttt{task638}): decide whether a dialogue turn was spoken by the user or the system (User / System).
    \item \textbf{Twitter emotion} (\texttt{task512}): classify the emotion of a tweet (anger / fear / joy / love / sadness / surprise).
\end{itemize}

We also report $K=4$, merging a subset of four tasks (MNLI neutral, word count, ATOMIC oWant and SemEval humour edit). 

Results are shown in Appendix~\ref{llm-lora-public}.

\newpage

\section{Results}
\subsection{Vision: CLIP ViT-B/32 Results}\label{vit}

\begin{figure}[!b]
    \centering
        \begin{subfigure}{.62\linewidth}
        \centering
        \includegraphics[width=\linewidth]{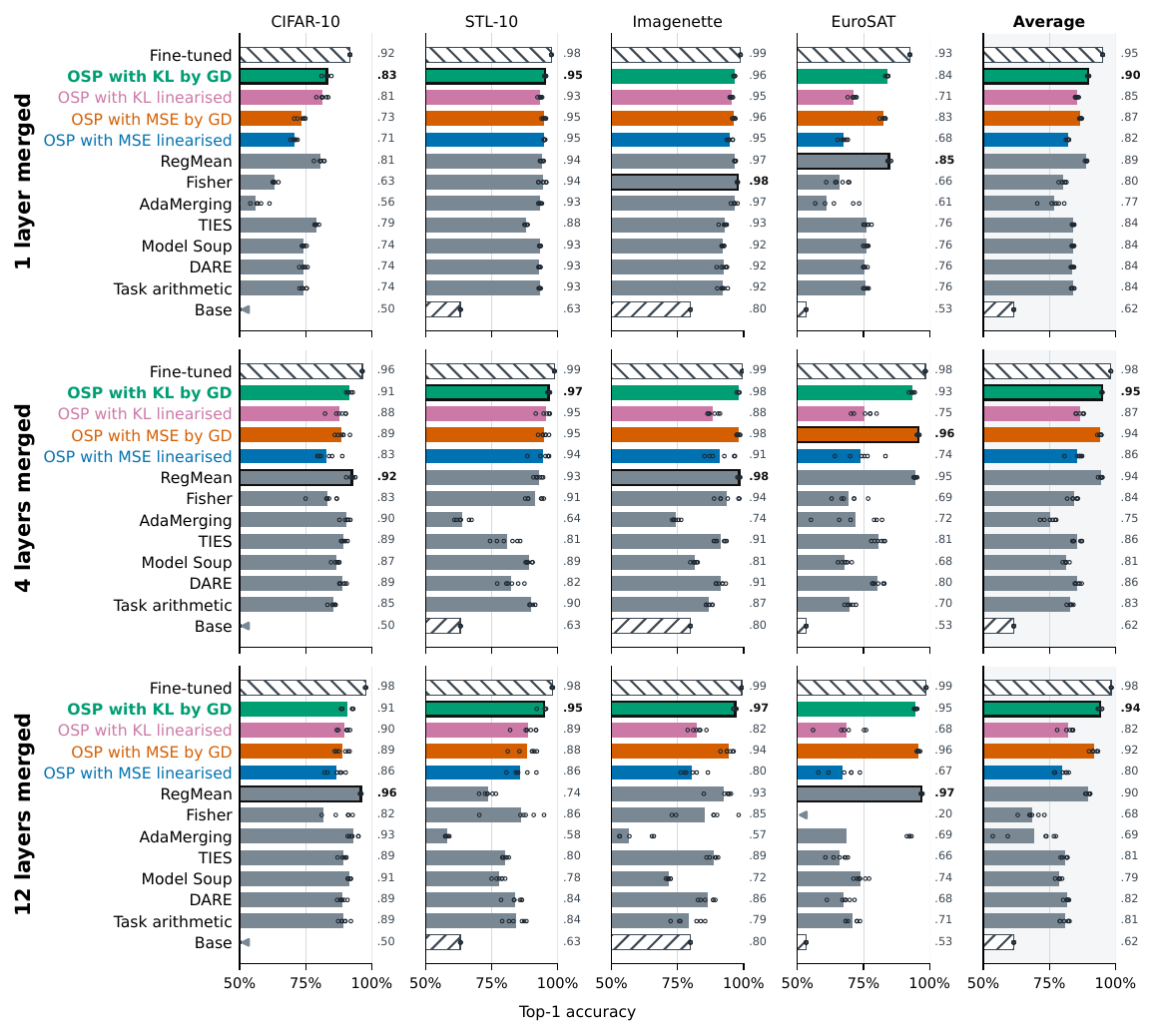}
        \caption{Accuracy on ViT-B/32.}
        \label{fig:acc_vitapp}
    \end{subfigure}\\
    \begin{subfigure}{\linewidth}
        \centering
        \includegraphics[width=.62\linewidth]{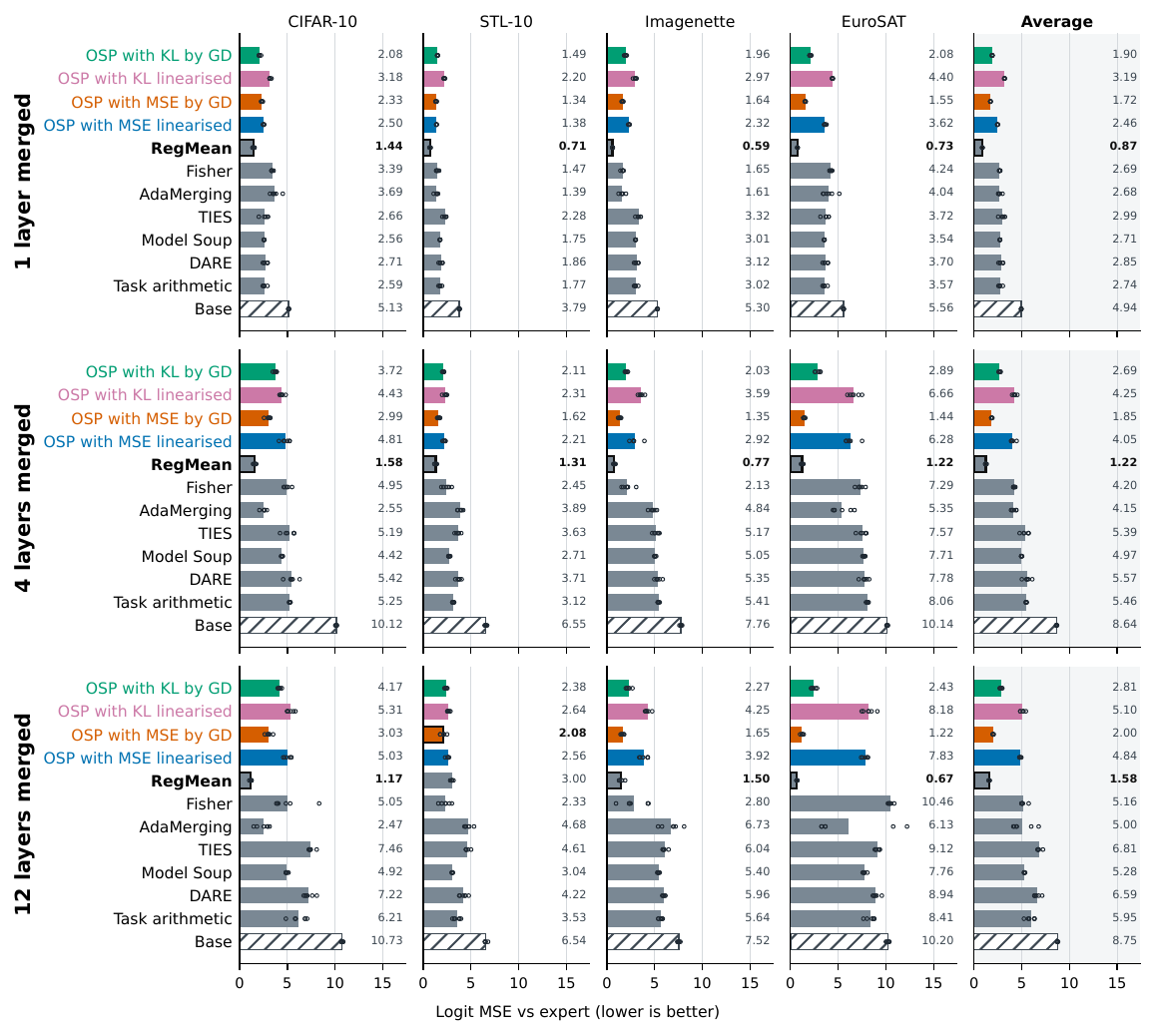}
        \caption{MSE on ViT-B/32.}
        \label{fig:mse_vit2app}
    \end{subfigure}
    \caption{Accuracy and MSE on ViT-B/32 across single-layer, 4-layer and 12-layer merging.}
    \label{fig:visionapp}
\end{figure}

\begin{figure}[!t]
    \centering
        \begin{subfigure}{.65\linewidth}
        \centering
        \includegraphics[width=\linewidth]{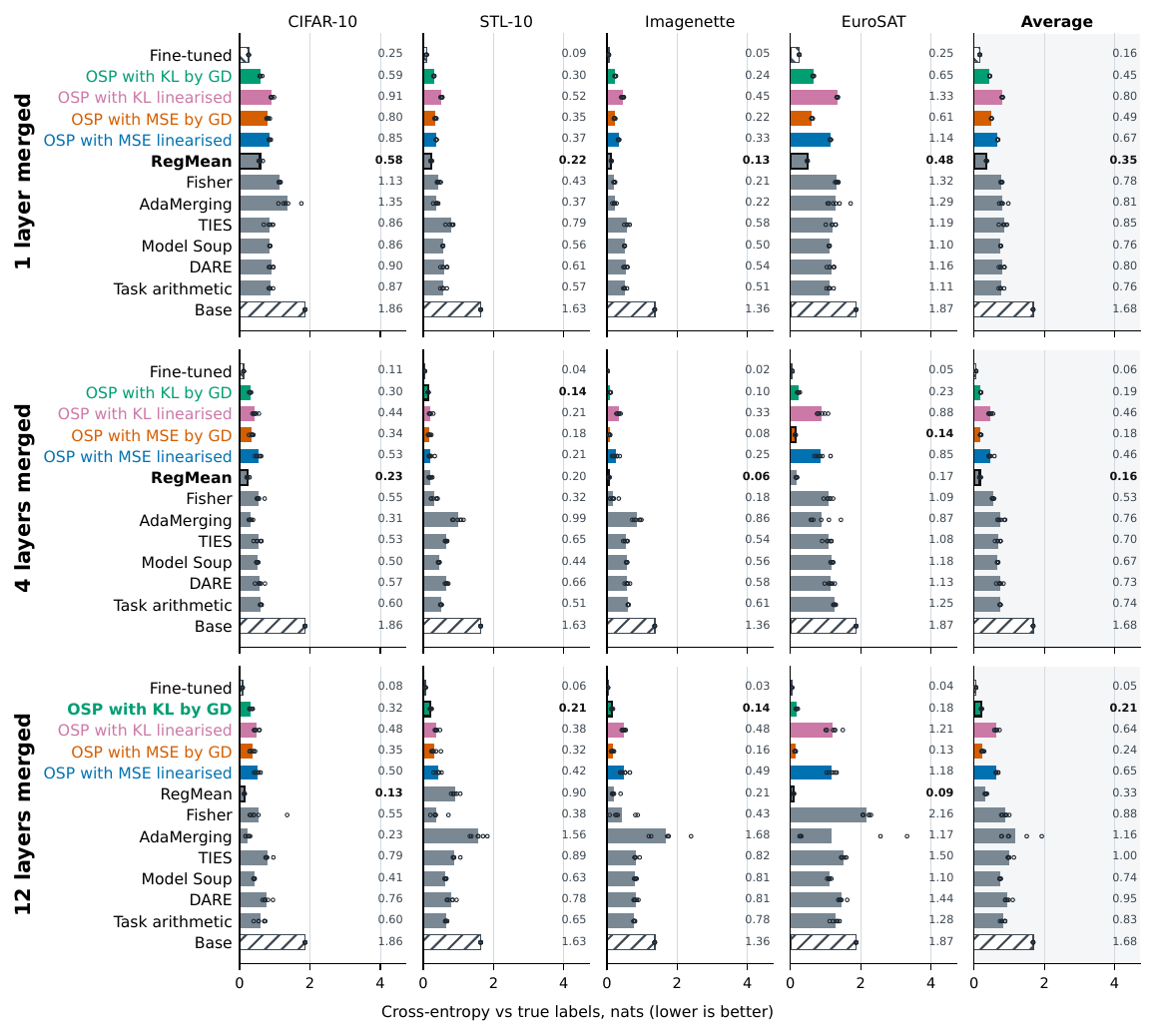}
        \caption{Cross-entropy on ViT-B/32.}
        \label{fig:ce_vitapp}
    \end{subfigure}\\
    \begin{subfigure}{\linewidth}
        \centering
        \includegraphics[width=.65\linewidth]{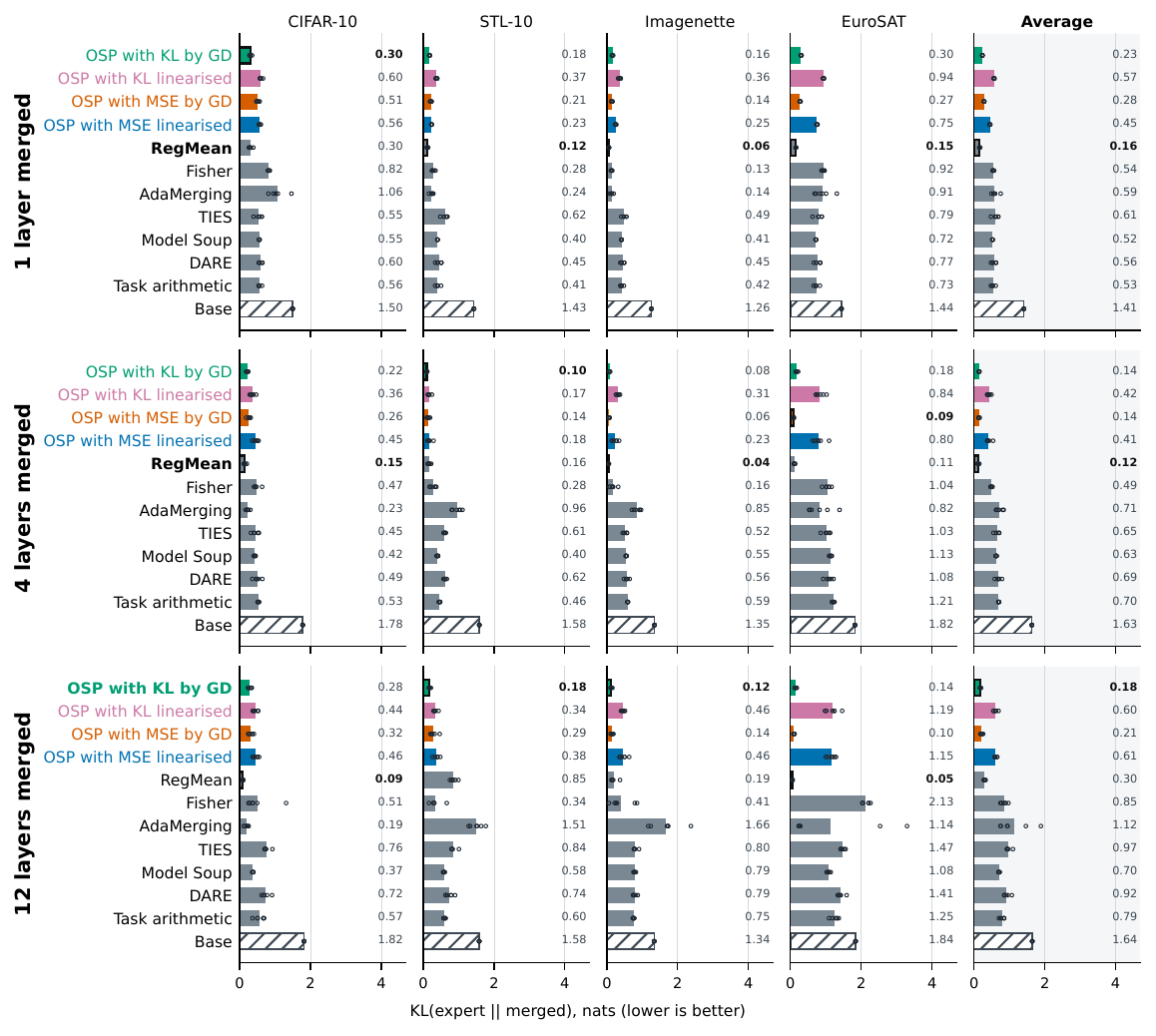}
        \caption{KL divergence on ViT-B/32.}
        \label{fig:kl_vit2app}
    \end{subfigure}
    \caption{Cross-entropy and KL divergence on ViT-B/32 across single-layer, 4-layer and 12-layer merging.}
    \label{fig:visionapp_cekl}
\end{figure}

\begin{figure}[!t]
        \centering
        \includegraphics[width=\linewidth]{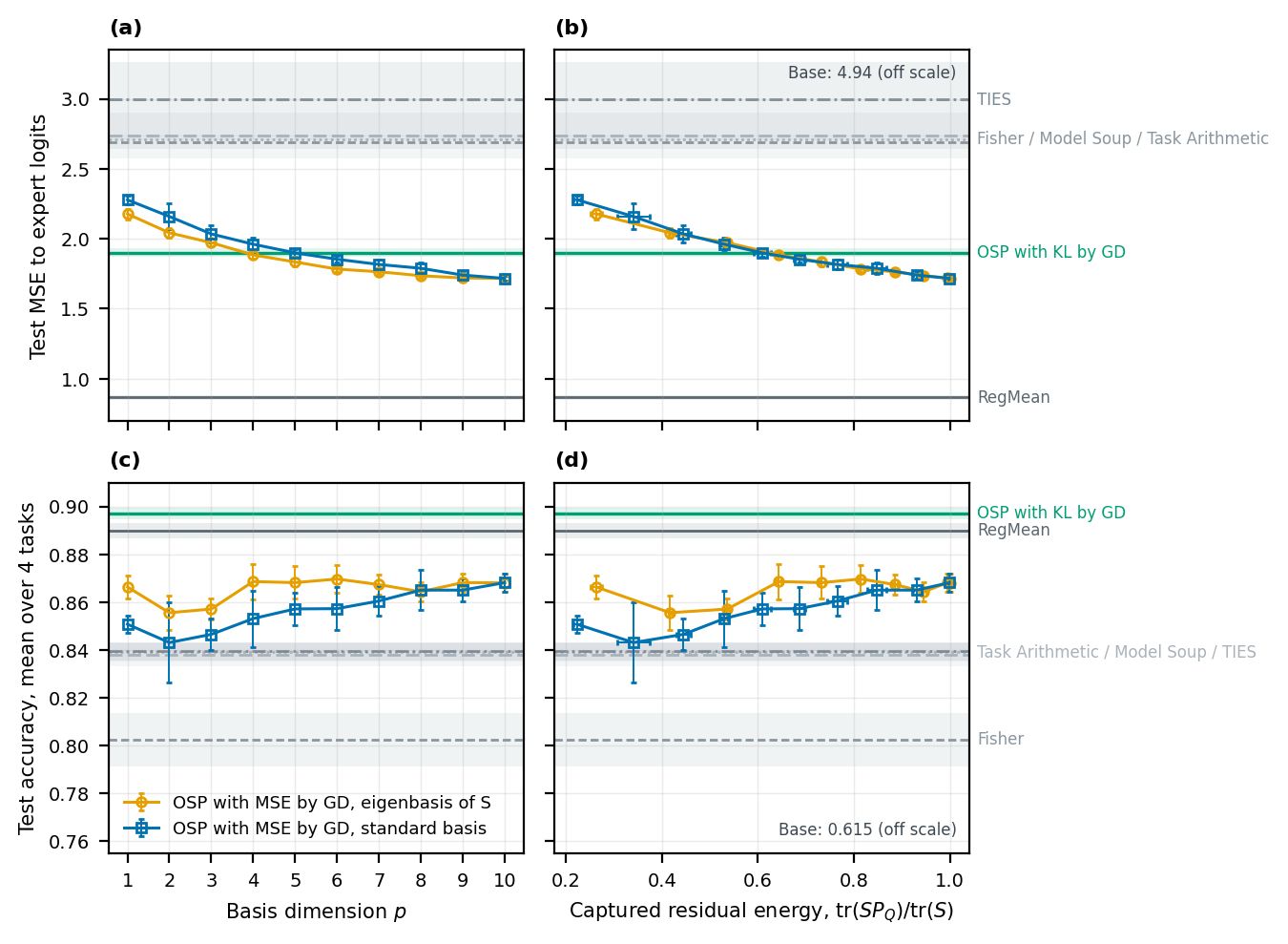}
        \caption{Basis and energy plots for ViT-B/32.}
        \label{energy}
\end{figure}

\clearpage

\subsection{LLM: Qwen3-4B-Base Fine-Tune Results}\label{llm}

\begin{figure}[!ht]
    \centering
    \begin{subfigure}{.62\linewidth}
        \centering
        \includegraphics[width=\linewidth]{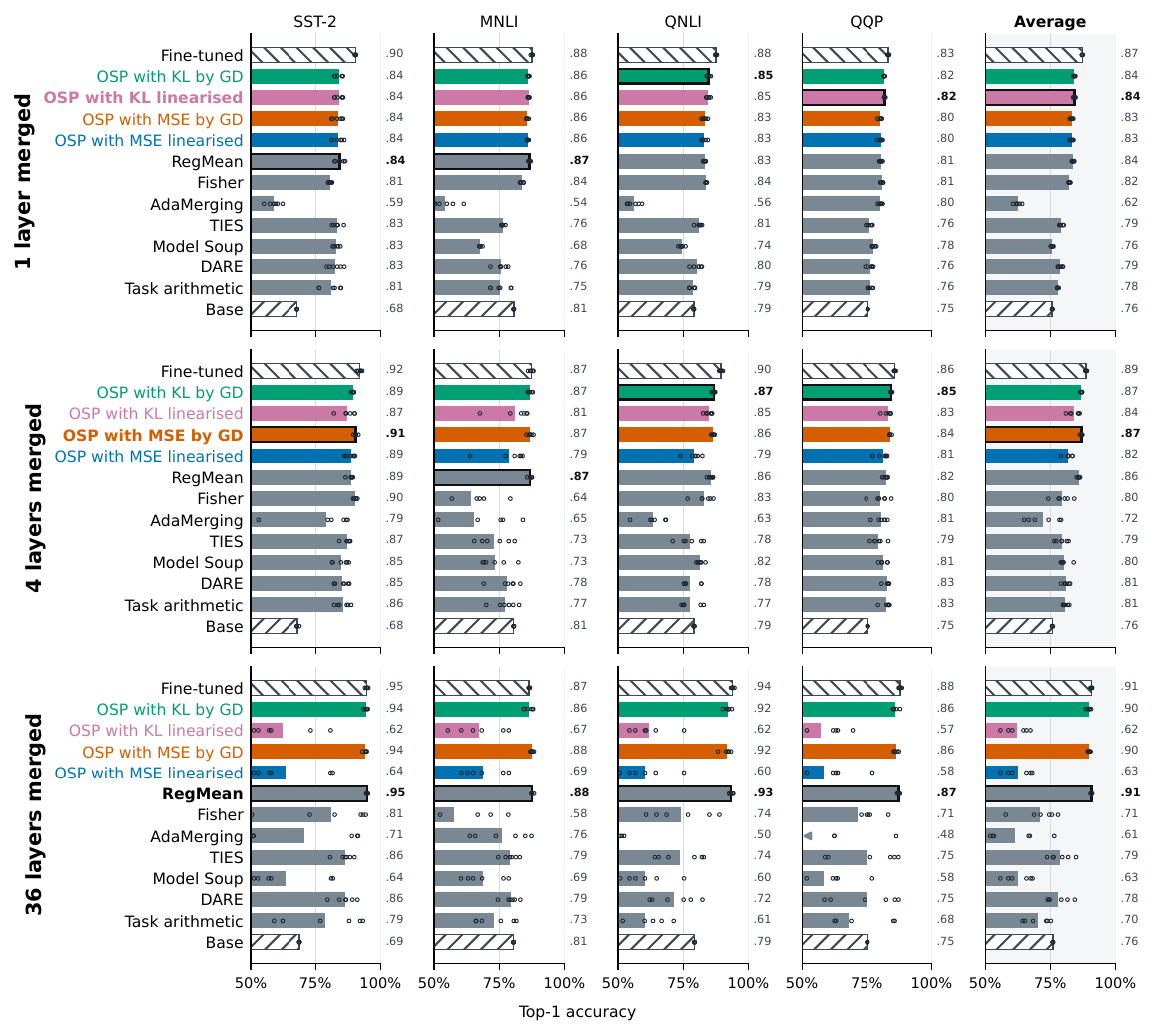}
        \caption{Accuracy on Qwen3-4B-Base.\\In the 36 layer case, linearised OSP with MSE does not improve on its initialisation, meaning results are identical to model soup.}
        \label{fig:acc_llmapp}
    \end{subfigure}\\
    \begin{subfigure}{\linewidth}
        \centering
        \includegraphics[width=.62\linewidth]{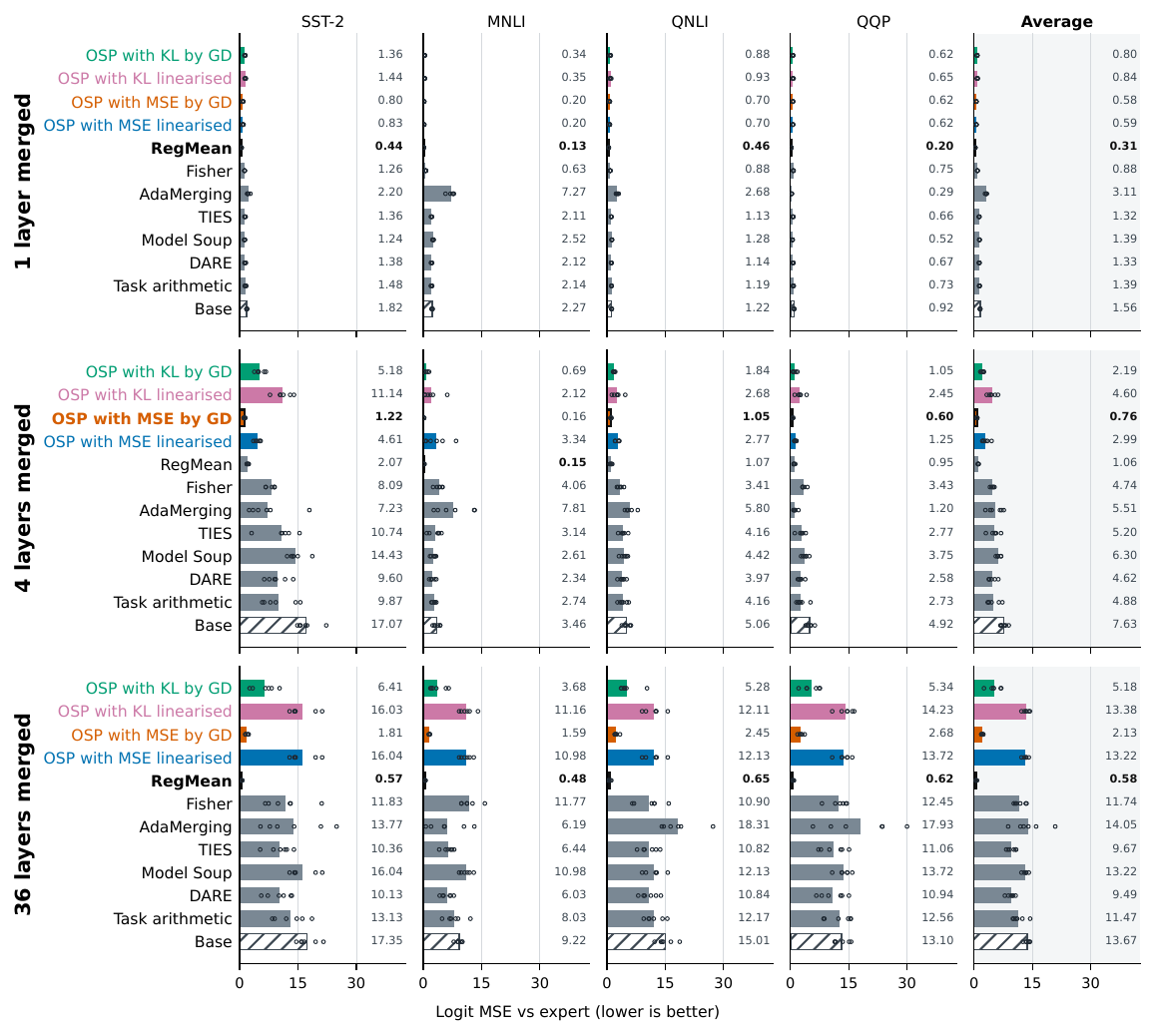}
        \caption{Logit MSE on Qwen3-4B-Base.}
        \label{fig:mse_llmapp}
    \end{subfigure}
    \caption{Accuracy and logit MSE on Qwen3-4B-Base, merging \texttt{down\_proj} in the last 1, 4 and all 36 layers (GLUE SST-2, MNLI, QNLI, QQP).}
    \label{fig:llmapp}
\end{figure}

\clearpage
\subsection{LLM LoRA: Qwen3-14B-Base Self-Trained LoRA Results}\label{llm-lora}

\begin{figure}[!ht]
    \centering
    \begin{subfigure}{\linewidth}
        \centering
        \includegraphics[width=.8\linewidth]{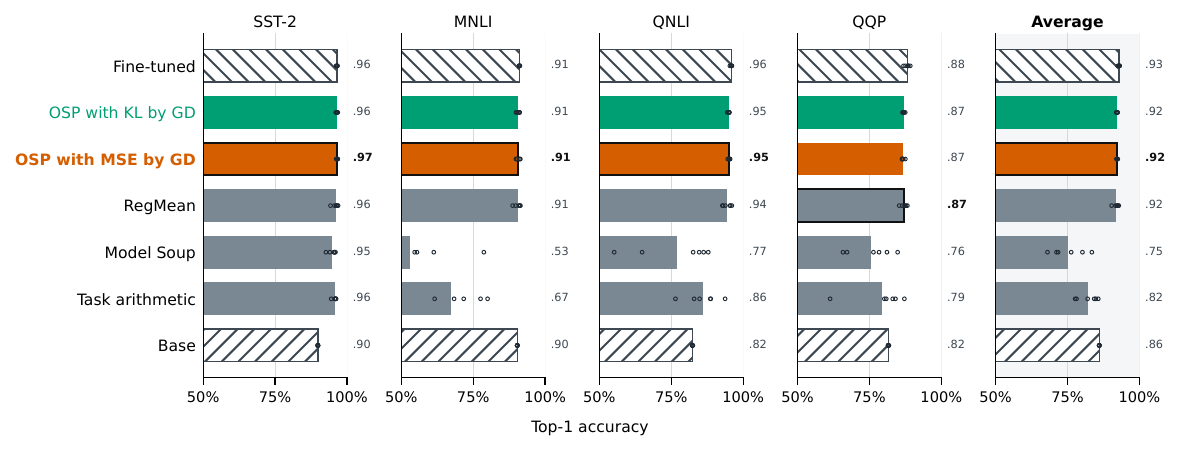}
        \caption{Accuracy on Qwen3-14B-Base.}
        \label{fig:acc_loraapp}
    \end{subfigure}\\[4pt]
    \begin{subfigure}{\linewidth}
        \centering
        \includegraphics[width=.8\linewidth]{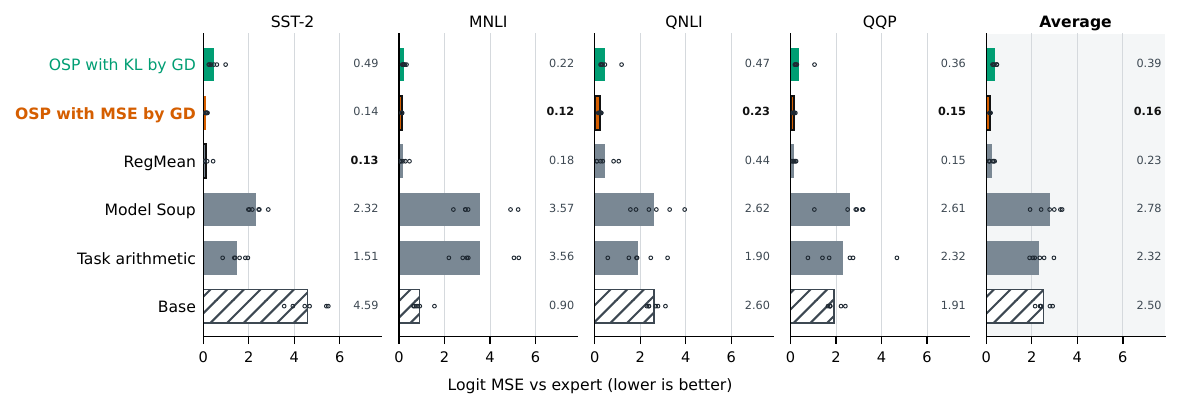}
        \caption{Logit MSE on Qwen3-14B-Base.}
        \label{fig:mse_loraapp}
    \end{subfigure}
    \caption{Accuracy and logit MSE when merging self-trained rank-16 LoRA experts on Qwen3-14B-Base.}
    \label{fig:loraapp}
\end{figure}

\clearpage
\subsection{LLM LoRA: Mistral-7B-Instruct LoRA Results}\label{llm-lora-public}

\begin{figure}[!ht]
    \centering
    \begin{subfigure}{\linewidth}
        \centering
        \includegraphics[width=\linewidth]{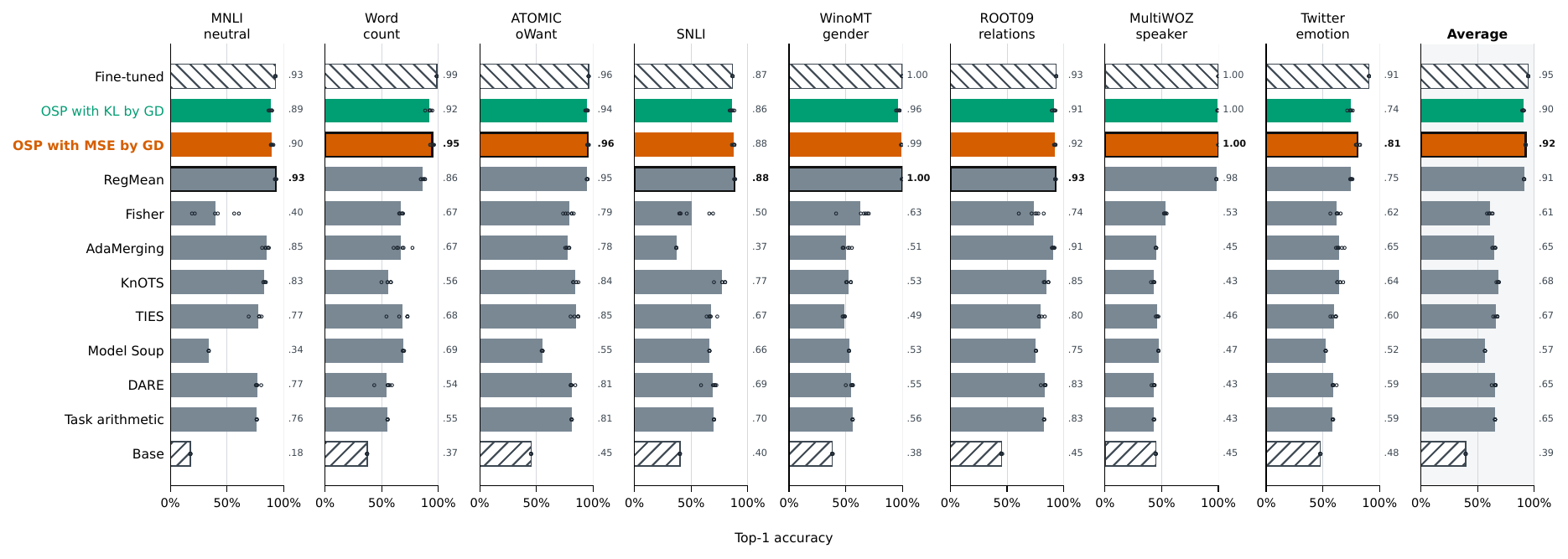}
        \caption{$K=8$ (the set in the main text).}
        \label{fig:public_k8app}
    \end{subfigure}\\[4pt]
    \begin{subfigure}{\linewidth}
        \centering
        \includegraphics[width=\linewidth]{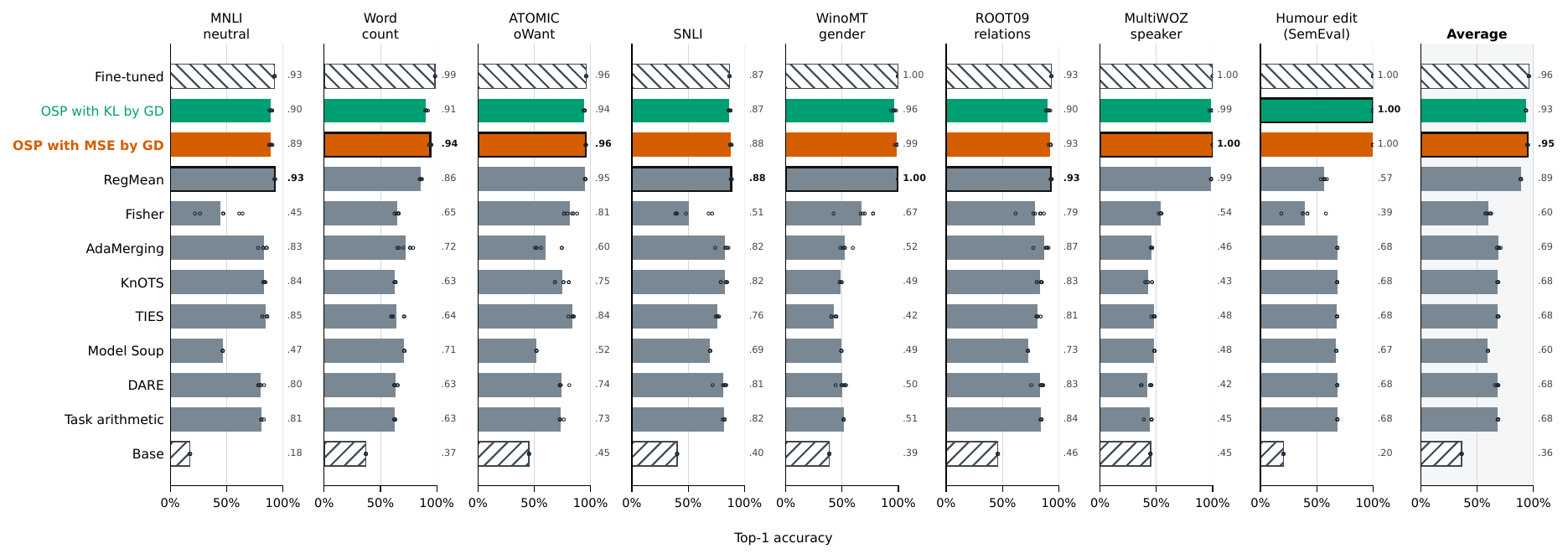}
        \caption{$K=8$ with the SemEval humour-edit task (\texttt{task496}) in place of Twitter emotion.}
        \label{fig:public_k8_496app}
    \end{subfigure}\\[4pt]
    \begin{subfigure}{.62\linewidth}
        \centering
        \includegraphics[width=\linewidth]{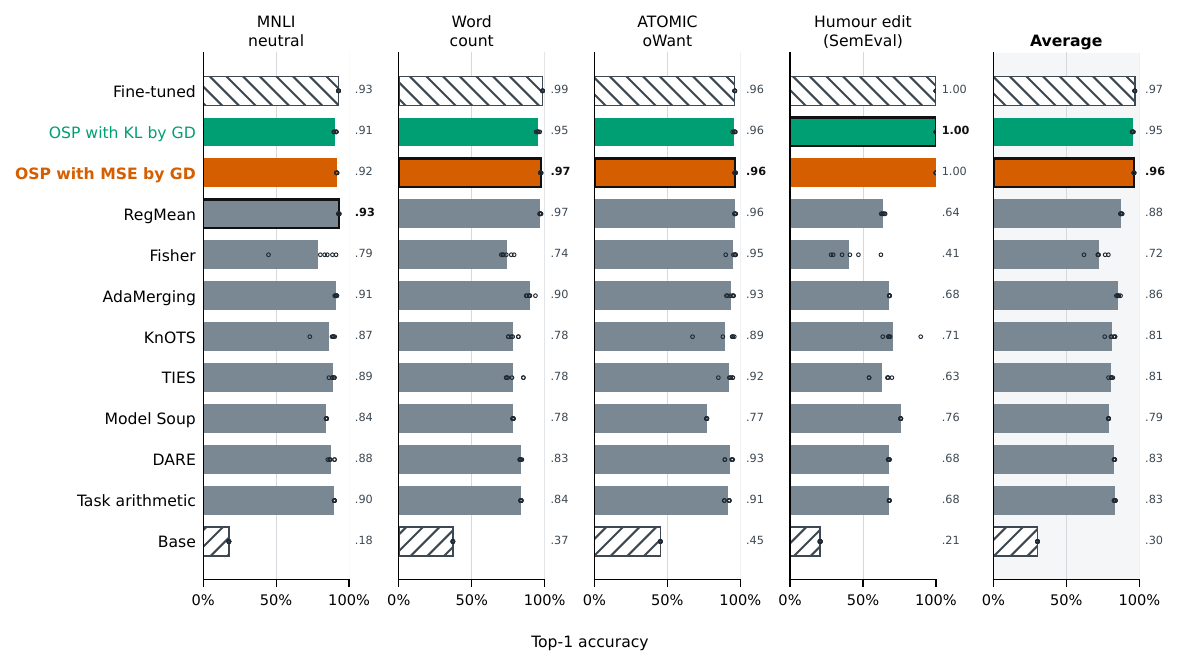}
        \caption{$K=4$.}
        \label{fig:public_k4app}
    \end{subfigure}
    \caption{Accuracy when merging public rank-16 LoRA adapters (Lots of LoRAs, \texttt{q,k,v} projections) on Mistral-7B-Instruct-v0.2.}
    \label{fig:publicapp}
\end{figure}

\clearpage

\end{document}